\documentclass[24pt]{article}
\usepackage[utf8]{inputenc}
\usepackage{amsmath,amsfonts,amssymb,amsthm, mathrsfs,mathtools}
\usepackage{centernot}
\usepackage{thm-restate}
\usepackage{float}
\usepackage{hyperref}
\usepackage{cleveref}
\usepackage{xcolor}
\usepackage{geometry}
\usepackage{algorithmic}
\usepackage[square,sort,comma,numbers]{natbib}
\crefname{algocf}{alg.}{algs.}
\Crefname{algocf}{Algorithm}{Algorithms}
\usepackage[ruled, linesnumbered]{algorithm2e}
\newtheorem{theorem}{Theorem}[section]
\newtheorem{lemma}[theorem]{Lemma}
\theoremstyle{definition}
\newtheorem{definition}[theorem]{Definition}
\newtheorem{claim}[theorem]{Claim}
\newtheorem{observation}[theorem]{Observation}
\newtheorem{lp}[theorem]{LP}

\newcommand{\cG}{A_h}
\newcommand{\ch}{h}
\newcommand{\cI}{\mathcal{I}}
\newcommand{\cF}{\mathcal{F}}
\newcommand{\cGU}{u_{p,\ch}}
\newcommand{\cGL}{l_{p,\ch}}
\newcommand{\cx}{x_{ap}}
\newcommand{\x}[1]{x^{(#1)}}
\newcommand{\cE}[2]{(#1,#2)}

\newcommand{\cD}{\mathcal{D}}
\newcommand{\cO}{\mathcal{O}}
\newcommand{\M}{M}
\newcommand{\X}{x}
\newcommand{\cy}{y}
\newcommand{\cz}{z}
\newcommand{\cw}{w}
\newcommand{\sOPT}{{\sf OPT}}
 
\newcommand{\maxmin}{maxmin }
\newcommand{\minmax}{mindom }
\DeclarePairedDelimiter{\norm}{\lVert}{\rVert}
\DeclarePairedDelimiter{\bnorm}{\big\lVert}{\big\rVert}
\DeclarePairedDelimiter{\Bnorm}{\Big\lVert}{\Big\rVert}
\DeclareMathOperator*{\argmin}{arg\,min}
\let\oldnl\nl
\newcommand{\nonl}{\renewcommand{\nl}{\let\nl\oldnl}}
\def\@fnsymbol#1{\ensuremath{\ifcase#1\or *\or \dagger\or \ddagger\or
   \mathsection\or \mathparagraph\or \|\or **\or \dagger\dagger
   \or \ddagger\ddagger \else\@ctrerr\fi}}

\title{Individual Fairness under Varied Notions of Group Fairness in Bipartite Matching - One Framework to Approximate Them All}

\author{Atasi Panda \\atasipanda@iisc.ac.in\\
Indian Institute of Science
\and Anand Louis \footnote{These two authors contributed equally}\\anandl@iisc.ac.in\\
Indian Institute of Science \and Prajakta Nimbhorkar \footnotemark[\value{footnote}] \\prajakta@cmi.ac.in\\
Chennai Mathematical Institute} 

\begin{document}

\maketitle

We study the probabilistic assignment of items to platforms that satisfies both group and individual fairness constraints. Each item belongs to specific groups and has a preference ordering over platforms. Each platform enforces group fairness by limiting the number of items per group that can be assigned to it. There could be multiple optimal solutions that satisfy the group fairness constraints, but this alone ignores item preferences. Our approach explores a `best of both worlds fairness' solution to get a randomized matching, which is ex-ante individually fair and ex-post group-fair. Thus, we seek a `probabilistic individually fair' distribution over `group-fair' matchings where each item has a `high' probability of matching to one of its top choices. This distribution is also ex-ante group-fair. Users can customize fairness constraints to suit their requirements. Our first result is a polynomial-time algorithm that computes a distribution over `group-fair' matchings such that the individual fairness constraints are approximately satisfied and the expected size of a matching is close to OPT. We empirically test this on real-world datasets. We present two additional polynomial-time bi-criteria approximation algorithms that users can choose from to balance group fairness and individual fairness trade-offs. 

For disjoint groups, we provide an exact polynomial-time algorithm adaptable to additional lower `group fairness' bounds. Extending our model, we encompass `maxmin group fairness,' amplifying underrepresented groups, and `mindom group fairness,' reducing the representation of dominant groups.'
\section{Introduction}
Matching is a foundational concept in theoretical computer science, well-studied over several years. Maximum bipartite matching finds applications in real-world scenarios, such as 
ad-auctions \citep{mehta_online_survey,mehta_online_adwords}, resource allocation \citep{halabian_resourceallocation}, scheduling~\citep{venkat_scheduling}, school choice~\citep{abdulkadiroglu2003,ControlledSchoolChoice}, 
and healthcare rationing~\citep{AB21,GaneshGHN23}. In this paper, we refer to the two partitions of the underlying bipartite graph as {\em items} and {\em platforms}. A matching is an allocation of items to platforms, allowing multiple assignments for each item and platform. Real-world items often have diverse attributes, leading to their categorization into different groups. To ensure equitable representation among these groups, it is natural to enforce {\em group fairness constraints} [\Cref{def:group_fairness}], which limit the number of items per group assigned to a platform by enforcing upper bounds. Also, one can specify lower bounds on the number of items from each group that need to be assigned to a platform, so as to 
ensure a minimum representation from each group among the items matched to a platform. 
The above constraints thus achieve \textit{Restricted Dominance} introduced in \citep{fair_clustering}, which asserts that the representation from any group on any platform does not exceed a user-specified cap, and \textit{Minority Protection} \citep{fair_clustering}, which asserts that the representation from any group, among the items matched to any platform is at least a user-specified bound [\Cref{def:strict_group_fairness}]. 

Both the definitions of group fairness are well-motivated by various applications like school choice, formation of committees in an organization, or teams to work on projects. For instance, in school choice, group-fairness constraints can promote diversity among students assigned to each school based on attributes like ethnicity and socioeconomic background, as observed in practical implementations \citep{case-study}. Similarly, in project teams, group fairness constraints ensure the inclusion of experts from all required fields.
Both definitions of group fairness are motivated by the \textit{Disparate Impact} doctrine \citep{disparate_impact} which broadly posits addressing unintentional bias which leads to widely different outcomes for different groups. 

However, since items have preferences over platforms, a matching meeting group fairness constraints alone may not be fair to individual items.
The exclusive use of group fairness constraints can lead to sub-optimal outcomes for individuals. Furthermore, deterministic algorithms for matching assign top choices to some individuals while assigning less preferred choices to others. This necessitates the introduction of {\em individual fairness constraints}. 
 In this paper, we consider probabilistic individual fairness constraints, first introduced in robust clustering \citep{anegg2020,Harris2019}. Instead of a single matching, the goal is to generate a distribution on group-fair matchings such that, in a matching sampled from the said distribution, the probability of each item being matched to
one of its top choices is within the user-specified bounds [\Cref{def:ind_fairness}]. 
Thus, this approach, known as the \textit{best of both worlds fairness} approach in literature, aims to compute an outcome with both ex-ante and ex-post fairness guarantees. 

\textit{In this paper, the central objective is to design efficient algorithms that compute an ex-ante probabilistic individually fair distribution over deterministic group-fair matchings.}

\subsection{An overview of our results and techniques}
Our approach revolves around formulating various notions of individual and group fairness using linear programming (LP). The key idea is to represent the LP's optimal solution as a convex combination of integer group-fair matchings, enabling the satisfaction of probabilistic individual fairness constraints through sampling. However, depending on the structure of the groups, it may not be possible to express the LP optimum as an exact convex combination of integral group-fair matchings. Nonetheless, our algorithms express an approximate LP optimum as a convex combination of integer matchings.

Our technique leads to a unified framework for different group fairness notions beyond fixed upper and lower bound constraints on the number of items from each group that can be matched to a platform. Two such notions, referred to as {\em maxmin group fairness} and {\em mindom group fairness} in this paper, are discussed in a later section (see \Cref{sec:prelim}). In maxmin group fairness, the goal is to maximize the minimum number of items that get matched to any platform from any one group. In mindom group fairness, the goal is to minimize the maximum number of items that get matched to any platform from any one group.
Informally, both these notions aim to get a matching with nearly equal representation from all groups. (See Section~\ref{sec:prelim} Definitions \ref{def:maxmin-group} and \ref{def:mindom-group} for formal definitions). In a similar spirit, for individual fairness, one can aim to provide the strongest possible guarantee simultaneously to all individuals in terms of the probability of being matched. We refer to this as {\em maxmin individual fairness} [\Cref{def:maxmin-ind}].

\section{Related Work}
Several allocation problems like resource allocation \citep{halabian_resourceallocation}, kidney exchange programs \citep{KidneyExchange}, school choice \citep{abdulkadiroglu2003}, candidate selection \citep{CandidateSelection}, summer internship programs \citep{SummerInternship}, and matching residents to 
hospitals \citep{Hospital-Resident} are modeled as matching problems. \citep{ManloveDavid} extensively examines preference-based matching in the stable marriage and roommates problems, hospitals/residents matching, and the house allocation problem. Since the people/items to be matched may belong to different groups, bipartite matchings under various notions of group fairness have been studied and their significance has been emphasized in literature \citep{vishnoi_fairness,luss_leximin_fair,devanur_ranking,CHRG16,halevi_fair_allocation,KMM15,BCZSK16}. \citep{AzizBiroYokoo2022} survey the developments in the
field of matching with constraints, including those based on regions, diversity, multi-dimensional capacities, and matroids. The fairness constraints are captured by upper and lower bounds \citep{ClassifiedStableMatching,IsraeliGapYear}, justified envy-freeness \citep{abdulkadiroglu2003}, or in terms of proportion of the final matching size \citep{CandidateSelection}. Historically, discriminated groups in India are protected with vertical reservations implemented as set-asides, and other disadvantaged groups
are protected with horizontal reservations implemented as minimum guarantees(lower bounds) \citep{IndiaReservation}. 

In some applications, the items could belong to multiple groups as well. \citep{GroupFairMatching} present a polynomial-time algorithm with an approximation ratio of $\frac{1}{\Delta+1}$ where each item belongs to at most $\Delta$ laminar families of groups per platform, and \citep{Rank-MaximalAndPOpular} show the NP-hardness of the problem without a laminar structure. While both papers focus only on group-fairness upper bounds, \citep{DiversityMatching} primarily focus on proportional diversity constraints with an emphasis on lower bounds in the general context.  However, group fairness constraints alone do not account for individual preferences. Our work aims to introduce individual fairness considerations into the problem and explore both upper and lower bounds for specific scenarios.

The notions of maxmin individual fairness, maxmin group fairness, and mindom group fairness are motivated by existing literature. Maxmin individual fairness, originally termed as the "distributional maxmin fairness" framework in \citep{IndividualFairness}, was further explored in group-fair ranking problems by \citep{MaxMinIndividualAndGroupRanking}. Their distribution is only over maximum matchings, and we extend this idea to a distribution over maximum group-fair matchings and a stronger notion of individual fairness. Maxmin group fairness is a natural extension of {\em Maxmin fairness}, initially introduced as a network design objective by \citep{Bertsekas1986DataN}(Section 6.5.2) and extensively studied in various areas of networking \citep{network_unified_framework,network_round_robin}. Mindom group fairness has been studied in network load distribution \citep{network_lexicography}, transmission cost sharing \citep{network_transmission}, and other network applications \citep{network_unified_framework}.
This concept of probabilistic individual fairness also has applications in fair-ranking \citep{MaxMinIndividualAndGroupRanking,fair_ranking_sampling} and graph-cut problems \citep{Dinitz2022}.

 \citep{FairnessInUncertainty} study individual fairness in ranking under uncertainty, extending fairness definitions by explicitly modeling incomplete information. Their approach mirrors \citep{Racke08} but assumes a posterior distribution over candidates' merits. Similarly, we assume that the individual and group fairness parameters are given. They express a distribution $\pi$ over rankings as a bistochastic matrix, where each entry denotes the probability of a candidate's position under $\pi$. They use an LP to maximize utility while enforcing fairness constraints and ensure marginal probabilities form a doubly stochastic matrix.
 The optimal solution is then decomposed as a
distribution over rankings using the Birkhoff-von Neumann algorithm \citep{birkhoff1946}. Though we use a similar approach for Theorems \ref{thm:approx_1_informal}, \ref{thm:approx_2_informal} and \ref{thm:exact_algo}, we have both group and individual fairness constraints. Our marginal probabilities do not form a doubly stochastic matrix, so we cannot use the Birkhoff-von Neumann decomposition \citep{birkhoff1946} directly, and hence need a different approach for \Cref{thm:approx_2_informal}. Similar to our group and individual fairness constraints (Definitions \ref{def:strict_group_fairness} and \ref{def:ind_fairness}), \citep{fair_ranking_sampling} address a related problem in fair ranking, particularly addressing laminar set structures using techniques akin to the Birkhoff-von Neumann decomposition.  

Fairness constraints with bounds on the number of items with each attribute are also studied in ranking and multi-winner voting \citep{fair_ranking,fair_multivote}.
Among other notions of fairness, \citep{Suhr2019} propose a fairness notion for ride-hailing platforms that distributes fairness over time, ensuring benefits proportional to drivers' platform engagement duration. Kletti et al. \citep{Kletti_2022} present an algorithm for optimizing rankings to maximize consumer utility while minimizing producer-side individual exposure unfairness.
\citep{MaxMinIndividualAndGroupRanking} explore maxmin fair distributions in general search problems with group fairness constraints, while \citep{RawlsianGroupAndIndividualFairness} examine Rawlsian fairness(maxmin fairness) in online bipartite matching, considering both group and individual fairness.
\citep{RawlsianGroupAndIndividualFairness} simultaneously address two-sided fairness but treat group and individual fairness separately. In contrast, we handle both individual and group fairness on the item side within a single bipartite matching instance, which has not been explored in the existing literature to the best of our knowledge. 

Our solution fits into the best of both worlds (BoBW) fairness paradigm, which is gaining attention in the fair allocation of indivisible items \citep{BoBW_Micha,BoBW_Uriel,BoBW_Vaish,BoBW_Aziz}.  In literature, popular target fairness properties have been envy-freeness, envy-freeness up to one item \citep{BoBW_Vaish,BoBW_Aziz}, proportionality, and proportionality up to one item \citep{BoBW_Micha,HSV23,HN23}. Other than these, \citep{BoBW_Uriel} study \textit{truncated proportional share}. \citep{BoBW_Vaish} showed that ex-ante envy-freeness (EF) and ex-post envy-freeness up to one item (EF1) BoBW outcomes are achievable for any allocation problem instance. \citep{BoBW_Micha} studied BoBW outcomes based on envy-based fairness in allocating indivisible items to agents with additive valuations and weighted entitlements. \citep{BoBW_Uriel} approach BoBW fairness from a fair-share
guarantee perspective. While our target fairness properties are group fairness and probabilistic individual fairness, our technique, in essence, resembles that of \citep{BoBW_Aziz}, where a randomized EF allocation is first generated and then decomposed as the convex combination of EF1 deterministic allocations. 

Fairness constraints, with various notions of fairness, have been considered in preference-based matchings, e.g. for kidney-exchange \citep{KidneyExchange}, for 
rank-maximality and popularity \citep{Rank-MaximalAndPOpular}, stability \citep{ClassifiedStableMatching}, stability under matroid constraints \citep{FleinerK16}, and in various settings of two-sided matching markets \citep{BeyhaghiTardos}, \citep{PatroCGG20}, \citep{FairMatchingsHuang}.
\section{Preliminaries}\label{sec:prelim}
{\bf Our problem:} The input instance consists of a bipartite graph denoted as $G=(A\cup P,E)$. Here
$A$ denotes the set of items and $P$ is the set of platforms. There is an edge, $(a,p) \in E$ if $a$ can be assigned to $p$. The items are grouped into possibly non-disjoint subsets $A_1, A_2, \ldots, A_{\chi}$ for an integer $\chi \geq 1$ such that $\cup_{h \in [\chi]} A_h = A$. Here $\chi$ denotes the total number of groups. Let $|A|=n$, $|P| = m$, $\Delta$ denote the maximum number of distinct groups to which any item belongs, and $N(v)$ denote the neighborhood of any node $v \in A \cup P$. Each item $a\in A$ has a preference list $R_a$, which contains a ranking of platforms, and let $R_{a,k}$ denote the set of top $k$ preferred platforms of $a$.

We define the group fairness and individual fairness notions below, these constraints are also part of the input.
\begin{definition}[\bf Group fairness]\label{def:group_fairness}Each platform, $p$, has upper bounds, $\cGU$, for all $h \in [\chi]$ denoting the maximum number of items from group $\ch$ that can be assigned to $p$. These are referred to as {\em group fairness constraints} in this paper.
For each $h \in [\chi]$, let $E_{p,\ch}$ denote the set of edges $\{(a,p): a \in A_h\}$. A matching $M\subseteq E$ is said to be {\em group-fair} if and only if
\begin{equation}\label{eq:grp_fairness}
    |E_{p,\ch} \cap M| \leq \cGU \text{ } \forall p \in P, h \in [\chi].
\end{equation}
\end{definition}

\noindent This notion of group fairness is also known as \textit{Restricted Dominance}, introduced in \citep{fair_clustering}.

\begin{definition}[\bf Strong group fairness]\label{def:strict_group_fairness} Along with upper bounds, each platform, $p$, has lower bounds, $\cGL$, for all $h \in [\chi]$ denoting the minimum number of items from group $\ch$ that should be assigned to $p$. A matching $M\subseteq E$ is said to be {\em strong group-fair} if and only if
\begin{equation}\label{eq:str_grp_fairness}
    \cGL \leq |E_{p,\ch} \cap M| \leq \cGU \text{ } \forall p \in P, h \in [\chi].
\end{equation}
\end{definition}

\noindent This notion of group fairness encompasses \textit{Minority Protection}, also introduced in \citep{fair_clustering}, along with \textit{Restricted Dominance}.

\begin{definition}[\bf Probabilistic individual fairness]\label{def:ind_fairness} In addition to the group fairness constraints, the input also contains {\em individual fairness parameters}, $L_{a,k}, U_{a,k} \in [0,1]$ for each item $a$ and $k \in [m]$. A distribution $\cD$ on matchings in $G$ is {\em probabilistic individually fair} if and only if $\forall a\in A, k \in [m]$
\begin{equation}\label{eq:IndFairness}
\begin{aligned}
    L_{a,k} \leq \Pr_{M \sim \cD}[\exists p\in R_{a,k} \textrm{ s.t. } (a,p)\in M]  \leq U_{a,k} \text{ }  
\end{aligned}
\end{equation}
    
\end{definition}

\noindent It is easy to see how \Cref{eq:IndFairness} can capture the requirement that items are matched to a high-ranking platform in their preference list with high probability and a low-ranking platform in their preference list with low probability. Our model allows users to set individual fairness constraints based on their requirements.

{\bf Objective: }Let $\cI=(G,A_1 \cdots A_{\chi}, \Vec{l}, \Vec{u}, \Vec{L},\Vec{U})$ denote an instance of our problem.
Our objective is to calculate a probabilistic individually fair distribution over a set of group-fair matchings, aiming to maximize the expected matching size when a matching is sampled from this distribution.

Note that our model provides a generic framework that accommodates various fairness settings, elaborated in \Cref{subsec:fair_notions}.

\subsection{Results}\label{sec:results}
We provide four different algorithms under different settings to compute a distribution over matchings. The support of the distribution is of size polynomial in the size of the instance, and is in fact of the same size as the number of iterations in the algorithms. Below, we list some known hardness results.

\subsubsection{Known hardness results}
Even without individual fairness constraints, finding a maximum size group-fair matching [\Cref{def:group_fairness}] is NP-hard \citep{Rank-MaximalAndPOpular}. Additionally, when there is a single platform and each item appears in at most $\Delta$ classes, the group fairness problem with only upper bounds is NP-hard to approximate within a factor of $\cO(\frac{\log^2 \Delta}{\Delta})$ \citep{GroupFairMatching}.

When an item can belong to multiple groups, determining if a feasible solution exists for group fairness constraints even with lower bounds alone is NP-hard \citep{DiversityMatching}, making the computation of a strong group-fair matching [\Cref{def:strict_group_fairness}] NP-hard. 

\subsubsection{Algorithmic results}
Our first contribution is an algorithm that computes a distribution over group-fair matchings such that the individual fairness constraints are approximately satisfied and the expected size of a matching is close to \sOPT. Throughout the paper, \sOPT represents the maximum expected size of a group-fair matching across all probabilistic individually fair distributions over such matchings.

\begin{theorem}[{\bf $O(\Delta\log n)$ bicriteria approximation} (Informal version of \Cref{thm:approx_2})]\label{thm:approx_2_informal}
    For any $\epsilon > 0$, there is a polynomial-time algorithm that outputs a distribution over group-fair matchings with the following properties: The expected size of a matching is at least $\frac{1}{f_{\epsilon}}\left(\sOPT + \epsilon \right)$, where $f_{\epsilon} = \cO(\Delta\log (n/\epsilon))$, and the individual fairness constraints are satisfied within additive and multiplicative factors of at most $\epsilon$ and $\frac{1}{f_{\epsilon}}$ respectively. Here $\Delta$ denotes the maximum number of classes an item belongs to. The algorithm reports infeasibility if no such distribution exists.
\end{theorem} 

\noindent For a platform $p$, let $g_p$ denote the number of distinct groups that have a non-empty intersection with $N(p)$, and let $g=max_{p\in P} g_p$.
 Next, we present an algorithm where the approximation guarantees are dependent on $g$, where each upper bound in the group fairness constraints is at least $g$.

\begin{theorem}[Informal version of \Cref{thm:approx_1}] \label{thm:approx_1_informal}  
    When all the group fairness upper bounds are at least $g$, there is a polynomial-time algorithm that computes a distribution over group-fair matchings, with the following properties: The expected size of a matching is at least $\frac{\sOPT}{2g}$, and the individual fairness constraints are satisfied up to a multiplicative factor of at most $\frac{1}{2g}$. The algorithm reports infeasibility if no such distribution exists.
\end{theorem}

Our next algorithm improves the multiplicative factor for violation of individual fairness constraints, and also the expected size of a matching, at the cost of an additive violation of group fairness constraints.

\begin{theorem}[Informal version of \Cref{thm:approx_3}] \label{thm:approx_3_informal} 
    When all the group fairness upper bounds are at least $g$, there is a polynomial-time algorithm that computes a distribution over matchings, with the following properties: The expected size of a matching is at least $\frac{\sOPT}{g}$, the matchings satisfy group fairness up to an additive factor of at most $\Delta$, and the individual fairness constraints up to a multiplicative factor of at most $\frac{1}{g}$. The algorithm reports infeasibility if no such distribution exists.
\end{theorem}

\noindent Table \ref{comparison-table} shows a comparison of the results in \Cref{thm:approx_2_informal}, \Cref{thm:approx_1_informal}, and \Cref{thm:approx_3_informal}. 
We give a polynomial-time exact algorithm for strong group fairness, when the groups are disjoint.
\begin{restatable}{theorem}{exactAlgo}\label{thm:exact_algo}
Given an instance of our problem where each item belongs to
exactly one group, there is a polynomial-time algorithm that either computes a probabilistic individually fair distribution over a set of \textbf{strong} group-fair matchings or reports infeasibility if no such distribution exists.
\end{restatable}


\begin{table*}[t]
\centering
\begin{tabular}{|c |c |c |c|} 
 \hline
   & \Cref{thm:approx_2_informal}(Algorithm \ref{alg4}) & \Cref{thm:approx_1_informal} & \Cref{thm:approx_3_informal} \\ 
 \hline
 Size-approximation & $\frac{1}{f_{\epsilon}}\left(\sOPT + \epsilon \right)$ & $\frac{\sOPT}{2g}$ & $\frac{\sOPT}{g}$ \\ 
 \hline
 Group Fairness Violation & None & None & $\Delta$-additive \\
 \hline
 Individual Fairness Violation  & $\frac{1}{f_{\epsilon}}$-multiplicative, $\frac{\epsilon}{f_{\epsilon}}$-additive & $\frac{1}{2g}$-multiplicative & $\frac{1}{g}$-multiplicative\\ [1ex]
 \hline
\end{tabular} 
\caption{Comparison of Approximation Algorithms. \\ 
$f_{\epsilon} = \cO(\Delta\log (n/\epsilon)$)}.
\label{comparison-table}
\end{table*}

\subsection{Extension to Other Fairness Notions}\label{subsec:fair_notions}
Our results can be extended to accommodate other fairness notions mentioned below.

\begin{definition}[\bf Maxmin individual fairness]\label{def:maxmin-ind} 
Let $\cD[a]$ denote $\Pr_{M \sim \cD}[\exists p \in P \text{ s.t } (a,p) \in M]$.
A distribution, $\cD$, over matchings is Maxmin individually fair if for all distributions $\cF$ over matchings and all $a \in A$,

\begin{align*}
    \cF[a] > \cD[a] 
    \implies \exists a' \in A \text{ s.t } \cD[a'] > \cF[a']
\end{align*}

\end{definition}

We refer to the goal of maximizing the representation of the worst-off groups as {\em \maxmin group fairness}, defined below. 
\begin{definition}[\bf Maxmin group fairness]\label{def:maxmin-group}
 Let $X^M_{\ch,p}$ denote the total number of items matched under a feasible matching, $M \subseteq E$, from group $\ch$ to platform $p$.
The matching $\M$ is said to be {\em \maxmin group-fair} if, for any other feasible matching $M'$, if $\exists p \in P, h \in [\chi]$ such that  $X^{M'}_{\ch,p} > X^M_{\ch,p}$, then there is some $p' \in P$, $h' \in [\chi]$ with $X^{M}_{\ch,p} \geq X^{M}_{\ch',p'}$ and $X^{M}_{\ch',p'} > X^{M'}_{\ch',p'}$. Here at least $p' \neq p$, or $h' \neq h$.
\end{definition}

In mindom group fairness, defined below, the goal is to minimize the representation of the most dominant groups. This is a dual to {\em \maxmin group fairness} 
\begin{definition}[\bf Mindom group fairness]\label{def:mindom-group}
Let $X^M_{\ch,p}$ denote the total number of items matched under a feasible matching, $M \subseteq E$, from group $\ch$ to platform $p$.
$M$ is said to be {\em \minmax group-fair} if, for any other feasible matching $M'$, if $\exists p \in P, h \in [\chi]$ such that  $X^{M'}_{\ch,p} < X^M_{\ch,p}$, then there is some $p' \in P$, $h' \in [\chi]$ with $X^{M}_{\ch,p} \leq X^{M}_{\ch',p'}$ and $X^{M}_{\ch',p'} < X^{M'}_{\ch',p'}$. Here at least $p' \neq p$, or $h' \neq h$.
\end{definition}

\subsubsection{Extension of Results}

\begin{theorem}\label{thm:ext_disjoint}
Given a bipartite graph with disjoint groups and a lower bound on the expected matching size, our framework and the polynomial-time algorithm from \Cref{thm:exact_algo} can be extended to compute the following:
\begin{enumerate}
    \item A probabilistic individually fair distribution over a set of maxmin or mindom group-fair matchings, with probabilistic individual fairness constraints.
    \item A maxmin individually fair distribution over strong group-fair matchings.
\end{enumerate}
\end{theorem}

\begin{theorem}\label{thm:ext}
Given a bipartite graph and a lower bound on the expected matching size, say $lb$, our framework and the polynomial-time algorithm from Theorems \ref{thm:approx_2_informal}, \ref{thm:approx_1_informal} and \ref{thm:approx_3_informal} can be extended to compute the following:
\begin{enumerate}
    \item A distribution over mindom group-fair or group-fair matchings, ensuring an expected matching size of at least $\frac{1}{f_{\epsilon}}\left(lb + \epsilon \right)$, with $f_{\epsilon}$ and the violation of probabilistic individual fairness or maxmin individual fairness as in \Cref{thm:approx_2_informal}.
    \item A distribution over mindom group-fair or group-fair matchings, guaranteeing an expected matching size of at least $\frac{lb}{2g}$, and a violation of probabilistic individual constraints or maxmin individual fairness by at most $\frac{1}{2g}$.
    \item A distribution over matchings, guaranteeing an expected matching size of at least $\frac{lb}{g}$, and a violation of probabilistic individual constraints or maxmin individual fairness by at most $\frac{1}{g}$. The mindom group-fairness or group-fairness is violated by an additive factor of at most $\Delta$.
\end{enumerate}
\end{theorem}
The proof of Theorems \ref{thm:ext_disjoint} and \ref{thm:ext}, and details of how to extend our results to these settings are in \Cref{appendix:extension}. 

\section{$O(\Delta \log n)$ bicriteria approximation algorithm}\label{sec:delta_approx}
In this section, our focus is on computing a probabilistic individually fair distribution over an instance of a bipartite graph $G = (A \cup P, E)$, where any arbitrary item, $a \in A$, can belong to at most $\Delta$ distinct groups. Our objective is to maximize the expected size of any matching sampled from this distribution while ensuring that the matching satisfies group fairness constraints. Within this context, we design a polynomial-time algorithm that provides an approximation factor dependent on $O(\Delta)$ and prove our main result, \Cref{thm:approx_2_informal}, formally stated below. 

\begin{theorem}[Formal version of \Cref{thm:approx_2_informal}]\label{thm:approx_2}
Given any $\epsilon > 0$, and an instance of our problem where each item can belong to at most $\Delta$ groups, there is a polynomial-time algorithm that computes a distribution $\cD$ over a set of group-fair matchings such that the expected size is at least $\frac{1}{f_{\epsilon}}\left(OPT + \epsilon \right)$ where $f_{\epsilon} = \cO(\Delta\log (n/\epsilon))$ and $n$ is the total number of items. Given the individual fairness parameters, $L_{a,k},U_{a,k} \in [0,1]$, for each item $a \in A$ and subset $R_{a,k}$ $\forall$ $k \in [n]$, \begin{align*}
      &\frac{1}{f_{\epsilon}}\left(L_{a,k} - \epsilon\right) \leq \Pr_{\M \sim \cD}[\exists p\in R_{a,k} \textrm{ s.t. } (a,p)\in M] \\ &\leq  \frac{1}{f_{\epsilon}}\left(U_{a,k} + \epsilon\right).
\end{align*}
\end{theorem}

Note that if we set $\epsilon = \displaystyle\min_{a \in A, k \in [n]}\frac{L_{a,k}}{2}$ in \Cref{thm:approx_2}, then $\forall a \in A, k \in [n]$,
\begin{align*}
         &\frac{L_{a,k}}{2f_{\epsilon}} \leq \frac{1}{f_{\epsilon}}(L_{a,k} - \epsilon) \leq \Pr_{\M \sim \cD}[\exists p\in R_{a,k} \textrm{ s.t. } (a,p)\in M] \\ &\leq  \frac{1}{f_{\epsilon}}(U_{a,k} + \epsilon) \leq \frac{3U_{a,k}}{2f_{\epsilon}}
\end{align*}

Therefore, we only get a multiplicative violation of individual fairness for $\epsilon = \displaystyle\min_{a \in A, k \in [n]}\frac{L_{a,k}}{2}$.

\subsection{Model Formulation}
We begin by formulating a Linear Programming (LP) model for our problem, specifically tailored to address \Cref{thm:approx_2}. A more extensive LP formulation, applicable to all our theorems, including \Cref{thm:exact_algo}, which integrates additional lower group fairness constraints, is detailed in \Cref{disjoint_groups}, \Cref{LP_disjoint}. Since items are assumed to be indivisible, we assume that the group fairness bounds are integers.

\begin{restatable}{lp}{primalIF}\label{LP_Primal_IF}
\begin{align}
&\max \sum_{(a,p) \in E} x_{ap} &\label{LP-1}\\
\text{such that}\quad &L_{a,k} \leq \sum_{p \in R_{a,k}}\cx \leq U_{a,k},  & \forall a \in A, \forall k \in [m] \label{LP-2}\\   
&\sum_{a \in A_h} \cx \leq \cGU , & \forall p \in P, \forall h \in [\chi] \label{LP-20}\\
&0 \leq \cx \leq 1 & \forall a \in A, \; \forall p\in P &\label{LP-21} 
\end{align}
\end{restatable}

\Cref{LP_Primal_IF} is a relaxation of the Integer Linear Programming formulation of the problem addressing group and individual fairness constraints. In the Integer Programming version, $x_{ap} = 1$ iff the edge connecting item $a$ to the platform $p$ is picked in the matching and $0$ otherwise. Constraints \ref{LP-2} and \ref{LP-20} capture the individual fairness and group fairness requirements, respectively.

Before delving into the algorithm, it is essential to note that \Cref{LP_Primal_IF} may become infeasible if the fairness constraints are inconsistent. This is possible due to the individual fairness constraints.
To address this, one solution is to introduce a variable to calculate the smallest multiplicative relaxation of the fairness constraints required to ensure the feasibility of \Cref{LP_Primal_IF}. This method is detailed in the section titled `Dealing with infeasibility' (\Cref{sec:LPinfeasibility}).

\subsection{Algorithm}

\begin{algorithm}[t]
\caption{$\cO(\Delta \log n)$-BicriteriaApprox$(\cI=(G,A_1 \cdots A_{\chi}, \Vec{l}, \Vec{u}, \Vec{L},\Vec{U}), \epsilon)$}
\label{alg4}
\nonl \textbf{Input} : $\cI$, $\epsilon$ \\
\nonl \textbf{Output} : Distribution over matchings satisfying the guarantees in \Cref{thm:approx_2}.\\
Solve \Cref{LP_Primal_IF} on $G$ with the parameters in the input instance, $\cI$, and store the result in $\X$. \label{4step:SolveLP}\\
$i \leftarrow 0, \alpha^{(0)} \leftarrow 0, G^{(0)} \leftarrow G, \X^{(0)} \leftarrow \X, sum \leftarrow 0, \cD \leftarrow \phi$\\
\While{$\norm{\X^{(i)}}_1 \geq \epsilon$}{ \label{4step:loop}
$i \leftarrow i+1, G^{(i)} \leftarrow G^{(i-1)} - \{\cE{a}{p} \; \lvert \; \x{i-1}_{ap} = 0\}$ \label{4step:graph}\\
Greedily find a Maximal Matching $\M^{(i)}$ in $G^{(i)}$ such that constraints $(\ref{LP-20})$ are not violated. \label{4step:greedy} \\
$\alpha^{(i)} \leftarrow \min_{\cE{a}{p}\in \M^{(i)}} \{\x{i-1}_{ap}\}$ \label{4step:alpha}\\
$sum \leftarrow sum + \alpha^{(i)}, \cD \leftarrow \cD \cup (\M^{(i)}, \alpha^{(i)})$ \label{4step:sum}\\
$\X^{(i)} \leftarrow \X^{(i-1)} - \alpha^{(i)} \cdot \M^{(i)} $  \label{4step:subtracting}\\
}
\For{$D^{(i)} \in \cD$}{
$D^{(i)} \leftarrow (\M^{(i)}, \frac{\alpha^{(i)}}{sum})$
}
\If{$\cD == \phi$}{Return {\em `Infeasible'}}
Return $\cD$
\end{algorithm}

First, we informally describe our algorithm (Algorithm \ref{alg4}) and the intuition behind the same. The key idea in our approach is to express a feasible solution, $\X$, of \Cref{LP_Primal_IF} as a convex combination of integer group-fair matchings. Achieving this allows us to satisfy our probabilistic individual fairness constraints by sampling from the probability distribution corresponding to this convex combination. If the groups are not disjoint, as is the case in our problem, then it is not known whether such a convex combination exists. Therefore,  we show that $\frac{\X-\X^{\dag}}{f}$ can be written as a convex combination of integer group-fair matchings, where $\norm{\X^{\dag}}_1 < \epsilon$ for some $\epsilon > 0$ and $f = \cO(\Delta \log n)$. Algorithm \ref{alg4} computes such a convex combination. 
 
The algorithm begins with a feasible solution of \Cref{LP_Primal_IF} solved on the input instance $\cI$, denoted by variable $x$. At round $i$ of the while loop (step \ref{4step:loop}), $\mathbf{G^{(i)}}$ denotes the state of the input graph after the $i^{th}$ iteration. It is a graph where edges with a corresponding zero value in $x^{(i-1)}$ are discarded in Step \ref{4step:graph}. $\mathbf{M^{(i)}}$ represents a group-fair maximal matching computed on $G^{(i)}$ in step \ref{4step:greedy}, and $\mathbf{x^{(i)}}$ denotes the state of $x$ after $i$ rounds. It is the residue after a scaled down $M^{(i)}$ is ``deducted'' from $x^{(i-1)}$ in step \ref{4step:subtracting}. $\mathbf{\alpha^{(i)}}$, denoting the minimum non-zero value associated with any edge in $x^{(i-1)}$ (step \ref{4step:alpha}), is used to scale down $M^{(i)}$ before ``deducting'' it from $x^{(i-1)}$ to ensure non-negative values in $x^{(i)}$. The algorithm terminates when the value of $||x^{(i)}||_1$  goes below $\epsilon$.
$\cD$ returned by Algorithm \ref{alg4} consists of tuples, and each tuple consists of a group-fair matching because $M^{(i)}$ is group-fair and its corresponding coefficient, $\frac{\alpha^{(i)}}{sum}$. If the loop terminates after $k$ rounds, $sum = \sum_{i=1}^k \alpha^{(i)}$. Clearly, $\sum_{i=1}^k \frac{\alpha^{(i)}}{sum} = 1$, therefore, $\cD$ is a distribution over group-fair matchings. 

One key intuition behind Algorithm \ref{alg4} is that in every iteration, we start with a solution, $\X^{(i-1)}$, that satisfies group-fairness constraints (\Cref{LP-20}), which allows us to greedily compute a group-fair matching $\M^{(i)}$ in the support of $\X^{(i-1)}$ (step \ref{4step:greedy} of Algorithm \ref{alg4}). This ensures that step \ref{4step:greedy} always returns a non-empty group-fair matching as long as $x^{(i-1)}$ has non-zero entries.

\noindent Next, we provide the proof of \Cref{thm:approx_2} using Algorithm \ref{alg4}.

\subsection{Proof of \Cref{thm:approx_2}}
The proof of \Cref{thm:approx_2} is based on a careful analysis of our simple (and fast) greedy algorithm (Algorithm \ref{alg4}). We first construct an LP formulation for our problem, concentrating solely on group-fairness constraints [\Cref{def:group_fairness}], excluding individual fairness constraints. This LP, denoted as \Cref{LP_Primal}, along with its dual counterpart \Cref{LP_Dual}, is introduced to facilitate our analysis. This choice is made because, instead of grounding our analysis on \Cref{LP_Primal_IF}, it suffices to focus on \Cref{LP_Primal}. \Cref{PrimalLPFeasibility} provides clarification on why that is.

\begin{restatable}{lp}{primal}\label{LP_Primal}
\begin{align}
&\max \sum_{(a,p) \in E} x_{ap} &\label{LP-9}\\
\text{such that}\quad &\sum_{a \in A_h} x_{ap} \leq \cGU \;, &\forall  h \in [\chi] ,\; \forall p \in P  &\label{LP-10}\\
& 0 \leq x_{ap} \leq 1 &\forall (a,p) \in E &\label{LP-11}
\end{align}
    
\end{restatable}

\begin{restatable}{lp}{dual}\label{LP_Dual}
\begin{align}
&\min \sum_{p \in P} \sum_{h \in [\chi]}\cGU w_{p,h} + \sum_{\cE{a}{p} \in E}y_{ap} &\label{LP-18}\\
\text{such that}\quad & 1 \leq \sum_{h: a \in A_h}w_{p,h} + y_{ap} \qquad \forall \cE{a}{p} \in E &\label{LP-19}
\end{align}
\end{restatable}

We show that the size of the matching $\M^{(i)}$, in the $i^{th}$ round of Algorithm \ref{alg4}, is at least $\frac{\norm{\X^{(i-1)}}_1}{\Delta+1}$ using dual fitting analysis technique \citep{williamson_shmoys_2011,vazirani2013approximation,greedy_dual_fitting} in \Cref{lem:deltaapprox}. We update the LP solution to $\X^{(i)}$ by ``removing'' $\alpha^{(i)}\M^{(i)}$ from $\X^{(i-1)}$ (step \ref{4step:subtracting} of Algorithm \ref{alg4}). $\alpha^{(i)}$ is the largest possible value such that the remaining LP solution is still a feasible solution of \Cref{LP_Primal} after step \ref{4step:subtracting}. 
Therefore, if a ``large'' mass of the LP solution remains in the $i^{th}$ iteration, i.e., $\norm{\X^{(i-1)}}_1$ is large, then we make ``large'' progress in the current iteration. This can essentially be used to show that $\sum_{i=1}^k \alpha^{(i)}$ is bounded by $f_{\epsilon} = 2(\Delta + 1)(\log (n/\epsilon) + 1)$ when $\norm{\X^{(i)}}_1 < \epsilon$ (\Cref{lem:alpha_bound}). Here, $k$ is the total number of iterations by Algorithm \ref{alg4}. Finally, setting $\hat{x} = \frac{\X - \X^{(k)}}{f_{\epsilon}}$, $t = f_{\epsilon}$, and $\delta = \frac{\epsilon}{f_{\epsilon}}$ in \Cref{lem:ind_fair}, proves the approximation guarantee on probabilistic individual fairness given by \Cref{thm:approx_2}. 

We first prove that any greedy maximal matching computed in step \ref{4step:greedy} of Algorithm \ref{alg4} is a $(\Delta+1)$-approximation of any feasible solution of \Cref{LP_Primal} using dual fitting analysis technique in Lemmas \ref{lem:feasiblity} and \ref{lem:deltaapprox}. First, let us look at the following observation.

\begin{observation}\label{PrimalLPFeasibility}
Any feasible solution of LP \ref{LP_Primal} augmented with constraint \ref{LP-2} is also a feasible solution of LP \ref{LP_Primal}
\end{observation}

\begin{lemma}\label{lem:feasiblity}
Let $M$ be a greedy maximal matching computed in step \ref{4step:greedy} of Algorithm \ref{alg4}. Let $\cy$ be an ordered set such that $\forall \cE{a}{p} \in E$, $y_{ap}$ is set to $1$ iff $M_{ap} = 1$, and $\cw$ be an ordered set such that  $\forall p \in P, h \in [\chi]$, $w_{p,h}$ is set to $1$ iff $\sum_{a \in A_h} M_{ap} = \cGU$. Then, $\cy$ and $\cw$ are a feasible solution of \Cref{LP_Dual}.
\end{lemma}

\begin{proof}
Let us fix an arbitrary edge, say $\cE{a}{p} \in E$. If $M_{ap}=1$, $y_{ap} = 1$ by definition. Therefore constraint $(\ref{LP-19})$ is satisfied. Let $C_{ap}$ denote a set of groups such that $a \in \cG$ and $A_h \cap N(p) \neq \phi$, $\forall \ch \in C_{ap}$, where $N(p)$ is the neighborhood of platform $p \in P$.
If $M_{ap} = 0$, we will show that there exists at least one group, say $\ch' \in C_{ap}$, such that  $w_{p,\ch'} = 1$. 
Suppose $\forall \ch \in  C_{ap}$, $w_{p,\ch} = 0$. This implies that $\forall \ch \in C_{ap}$, $\sum_{b \in \cG} M_{bp} < u_{p, \ch}$, by definition of $\cw$, therefore the edge, $\cE{a}{p}$ can be included in the matching without violating the group fairness constraint \ref{LP-10}, which is a contradiction since $\X$ is a maximal matching. Hence, there exists at least one group, say $\ch' \in C_{ap}$ such that $\sum_{a \in A_{h'}}M_{ap} = u_{p,\ch'}$, which in turn implies that there exists at least one group, $\ch' \in C_{ap}$, such that $w_{p,\ch'}=1$. Therefore, constraint $(\ref{LP-19})$ is not violated and $\cy$ and $\cw$ are a feasible solution to \Cref{LP_Dual}.
\end{proof}

\begin{lemma}\label{lem:deltaapprox}
If $M$ is a greedy maximal matching computed in step \ref{4step:greedy} of Algorithm \ref{alg4}, then $\displaystyle \sum_{(a,p) \in E}M_{ap} \geq \frac{\sum_{(a,p) \in E}\Psi_{ap}}{\Delta+1}$, where $\Psi$ is any feasible solution of \Cref{LP_Primal_IF}.
\end{lemma}

\begin{proof}
Let $\cy$ be an ordered set such that $\forall \cE{a}{p} \in E$, $y_{ap}$ is set to $1$ iff $M_{ap} = 1$, and $\cw$ be an ordered set such that  $\forall p \in P, h \in [\chi]$, $w_{p,\ch}$ is set to $1$ iff $\sum_{a \in A_h} M_{ap} = \cGU$. From \Cref{lem:feasiblity},  we know that $\cy$ and $\cw$ are a feasible solution of \Cref{LP_Dual}. Let $\hat{\psi}$ be the dual objective function evaluated at $\cy$ and $\cw$ and $\psi$ be the primal objective function evaluated at $M$. Note that by definition of $\cy$, $\sum_{\cE{a}{p}\in E} y_{ap}$ is equal to the number of edges in the maximal matching, which is $\psi$. Since $w_{p,\ch}$ is set to $1$ iff $\sum_{a \in A_h} M_{ap} = \cGU$, 
$$\sum_{p \in P} \sum_{h \in [\chi]}\cGU w_{p,h} = \sum_{p \in P} \sum_{h \in [\chi]}\sum_{a \in A_h} M_{ap}.$$ 
Since any item, say $a \in A$, can belong to at most $\Delta$ groups, any edge, $\cE{a}{p}$, such that $x_{ap} = 1$, can contribute to at most $\Delta$ many tight upper bounds, therefore, 
$$\sum_{p \in P} \sum_{h \in [\chi]}\sum_{a \in A_h} M_{ap} \leq \Delta\psi.$$
Hence, 
$$\hat{\psi} = \sum_{p \in P} \sum_{h \in [\chi]}\cGU w_{p,h} + \psi \leq \Delta  \psi + \psi = (\Delta +1)\psi.$$ 
Let $\psi^*$ and $\hat{\psi}^*$ be the optimal objective costs of \Cref{LP_Primal} and \Cref{LP_Dual}, respectively, since \Cref{LP_Primal} is a maximization, we get
\[ (\Delta+1)\psi \geq \hat{\psi} \geq \hat{\psi}^* \geq \psi^* \implies \sum_{(a,p) \in E}M_{ap} \geq \frac{\psi^*}{(\Delta+1)}\]
By \Cref{PrimalLPFeasibility}, $\psi^* \geq \sum_{(a,p) \in E}\Psi_{ap}$, therefore,
\[\sum_{(a,p) \in E}M_{ap} \geq \frac{\sum_{(a,p) \in E}\Psi_{ap}}{(\Delta+1)}\]
\end{proof}

In the rest of the section, $x^{(i)}$ denotes the state of $x^{(0)}$ after the $i^{th}$ round of the while loop in Algorithm \ref{alg4}, where $x^{(0)}$ is an optimal solution of \Cref{LP_Primal_IF} (Step \ref{4step:SolveLP}), $M^{(i)}$ denotes the greedy maximal group-fair matching computed in step \ref{4step:greedy} of the $i^{th}$ round and $\alpha^{(i)}$ denotes the coefficient being computed in step \ref{4step:alpha} of the $i^{th}$ round.
 
\begin{lemma}\label{Alg4-run-time}
The run-time of Algorithm \ref{alg4} is polynomial in the number of nodes, $|V|$, and the number of edges, $|E|$, of the input graph, $G$.
\end{lemma}
\begin{proof}
Let $i$ denote an arbitrary iteration of the while loop in Algorithm \ref{alg4}. Since $\alpha^{(i)} = \min_{\cE{a}{p}\in M^{(i)}} x^{(i-1)}_{ap}$, as seen in step \ref{4step:alpha}, at least one edge is removed from the support of the solution, in each iteration. Hence the norm can go to zero in $|E| = O(|V|^2)$ iterations, and since the algorithm exits once $\norm{\X}_1 < \epsilon$, it runs for at most $|E|$ rounds, therefore, the while loop in Algorithm \ref{alg4} terminates in $O(|V|^2)$ time. 

The LP in step \ref{4step:SolveLP} (\Cref{LP_Primal_IF}) has $|E|$ variables and $2nm + \chi m $ constraints where $\chi$ is the total number of groups, $n$ is the total number of items and $m$ is the total number platforms in the input instance, $\cI$. $|V| = n+m$, therefore, the runtime of \Cref{LP_Primal_IF} and, as a result, that of Algorithm \ref{alg4} is polynomial in the number of nodes, $|V|$, and the number of edges, $|E|$, of the input graph, $G$.
\end{proof}

\begin{observation}\label{obs:positive}
Let Algorithm \ref{alg4} terminate in $k$ iteration, then, $\forall i \in \{0\}\cup[k]$, $\cx^{(i)} \ge 0$, $\forall \cE{a}{p} \in E$.
\end{observation}
\begin{proof}
We will use induction to show this.
For the base case, $i=0$, since $\X$ is a feasible solution of LP \ref{LP_Primal_IF}, $\cx^{(0)} = \cx \ge 0$, $\forall \cE{a}{p} \in E$ because of constraint \ref{LP-21}. For the induction step let us assume that $\cx^{(i-1)} \ge 0$, $\forall \cE{a}{p} \in E$. Since no edge, say $(a,p)$, such that $\cx^{(i-1)} = 0$, will be picked in the maximal matching, $M^{(i)}$, $\cx^{(i)} = \cx^{(i-1)} - \min_{\cE{a}{p}\in M^{(i)}} \{ \cx^{(i-1)} \}$, iff $\cx^{(i-1)} \neq 0$. Therefore, $\cx^{(i)} \ge 0$, $\forall \cE{a}{p} \in E$.
\end{proof}

\begin{claim}\label{clm:norm_bound}
Let Algorithm \ref{alg4} terminate in $k$ iterations, then, $\forall i \in [k-1]$, $$\Bnorm{\sum_{j=i}^{k} \alpha^{(j)}M^{(j)}}_1 \leq \bnorm{\X^{(i-1)}}_1$$\end{claim}

\begin{proof}
If Algorithm \ref{alg4} terminates in $k$ rounds, from step \ref{4step:subtracting} of Algorithm \ref{alg4} we know that $\forall \cE{a}{p} \in E$, either $\cx^{(k)} = 0$ or
\begin{equation}\label{eq:init}
    \cx^{(k)} = \cx^{(k-1)} - \alpha^{(k)}M_{ap}^{(k)}
\end{equation}
Let us consider an integer, $i \le k-1$, then, recursively replacing $\cx^{(k-1)}$ on the RHS of \Cref{eq:init} until the index reaches $i-1$, we have
$$\cx^{(k)} = \cx^{(i-1)} - \sum_{j=i}^k\alpha^{(j)}M_{ap}^{(j)}$$
Since $\cx^{(k)} \ge 0$, $\forall \cE{a}{p} \in E$, from \Cref{obs:positive}, $\forall \cE{a}{p} \in E$, $\sum_{j=i}^k\alpha^{(j)}M_{ap}^{(j)} \le \cx^{(i-1)}$. Therefore, $\forall i \in [k-1]$,
$$\Bnorm{\sum_{j=i}^{k} \alpha^{(j)}M^{(j)}}_1 \leq \bnorm{\X^{(i-1)}}_1$$
\end{proof}

\begin{lemma}\label{lem:alpha_bound}
Let $i_c$ denote the first iteration of the while loop in Algorithm \ref{alg4} such that $\bnorm{\X - \sum_{j=1}^{i_c} \alpha^{(j)} \cdot M^{(j)}}_1 < \frac{\norm{\X}_1}{2^c}$, then 
$\sum_{j=1}^{i_c} \alpha^{(j)} \le 2c(\Delta+1)$.
\end{lemma}

\begin{proof}
Let Algorithm \ref{alg4} terminate in $k$ rounds. Since $\forall i \in \{0\} \cup[k]$, $x^{(i)}_{ap}\ge 0$ $\forall \cE{a}{p} \in E$, from \Cref{obs:positive}, it is easy to see that $\alpha^{(i)} = \min_{\cE{a}{p}\in M^{(i)}} \{x_{ap}^{(i-1)}\} > 0$, $\forall i \in [k]$. Hence $\sum_{a \in A_h} x^{(i)}_{ap} \leq \cGU$ $\forall  h \in [\chi] , \forall p \in P$, and $\X^{(i)}$ is a feasible solution of \Cref{LP_Primal}.
Therefore, using \Cref{lem:deltaapprox}, we have
\begin{equation}\label{eqn:1}
    \norm{M^{(i)}}_1 \geq \frac{\bnorm{\X - \sum_{j=1}^{i-1} \alpha^{(j)} \cdot M^{(j)}}_1}{\Delta+1}
\end{equation}
Now, we will prove the Lemma by induction on $c$. Let's first look at the base case where $c=1$. By definition,
    $$\Bnorm{\X - \sum_{j=1}^{i_1} \alpha^{(j)} \cdot M^{(j)}}_1 < \frac{\bnorm{\X}_1}{2}$$
Since, $\forall j \le i_1$, $\norm{\X^{(j)}}_1 \ge \frac{\norm{x}_1}{2}$, from \Cref{eqn:1} we have $\norm{M^{(j)}}_1 \geq \frac{\norm{\X^{(j)}}_1}{\Delta+1} \ge \frac{\norm{\X}_1}{2(\Delta+1)}$, $\forall j \le i_1$. From \Cref{clm:norm_bound}, we know that $\bnorm{\sum_{j=1}^{k} \alpha^{(j)} \cdot M^{(j)}}_1 \le \norm{\X^{(0)}}_1 = \norm{\X}_1$. Since $i_1 \le k$, $\norm{\sum_{j=1}^{i_1} \alpha^{(j)} \cdot M^{(j)}}_1 \le \norm{\X}_1$. Therefore, 
$$
\norm{\X}_1 \geq \Bnorm{\sum_{j=1}^{i_1} \alpha^{(j)} \cdot M^{(j)}}_1 \ge \sum_{j=1}^{i_1} \alpha^{(j)} \cdot \frac{\norm{\X}_1}{2(\Delta+1)}
$$
\begin{equation}\label{eqn:3}
    \implies \sum_{j=1}^{i_1} \alpha^{(j)} \le 2(\Delta+1)
\end{equation}

For the induction step, let us assume that for some iteration $i_{c-1}$,
\begin{equation}\label{eq:induction}
    \sum_{j=1}^{i_{c-1}} \alpha^{(j)} \le 2(c-1)(\Delta+1)
\end{equation}
By definition, 
$\norm{\X - \sum_{j=1}^{i_c} \alpha^{(j)} \cdot M^{(j)}}_1 < \frac{\norm{\X}_1}{2^{c}}$, therefore, $\forall j \le i_c$, $\norm{\X^{(j)}}_1 \ge \frac{\norm{x}_1}{2^c}$. From \Cref{eqn:1}, we get $\forall j \le i_c$, $\norm{M^{(j)}}_1 \ge \frac{\norm{\X^{(j)}}_1}{\Delta+1} \ge \frac{\norm{\X}_1}{2^c(\Delta+1)}$. Therefore,
\begin{equation}\label{eqn:4}
    \Bnorm{\sum_{j=i_{c-1}+1}^{i_c} \alpha^{(j)} \cdot M^{(j)}}_1 \geq \sum_{j=i_{c-1+1}}^{i_c} \alpha^{(j)} \cdot \frac{\norm{\X}_1}{2^c(\Delta+1)}
\end{equation}

From \Cref{clm:norm_bound}, we know that $\bnorm{\sum_{j=i_{c-1}+1}^{k} \alpha^{(j)} \cdot M^{(j)}}_1 \le \norm{\X^{(i_{c-1})}}_1$. Since, $i_c \le k$,
\begin{align*}
    &\Bnorm{\sum_{j=i_{c-1}+1}^{i_c} \alpha^{(j)} \cdot M^{(j)}}_1 \leq \norm{\X^{(i_{c-1})}}_1 \\&= \Bnorm{\X - \sum_{j=1}^{i_{c-1}} \alpha^{(j)}M^{(j)}}_1 < \frac{\norm{x}_1}{2^{c-1}}.
\end{align*}
Therefore, using \Cref{eqn:4}, we have $\frac{\norm{x}_1}{2^{c-1}} > \sum_{j=i_{c-1}+1}^{i_c} \alpha^{(j)} \cdot \frac{\norm{\X}_1}{2^c(\Delta+1)}$, hence,
\begin{equation}\label{eqn:5}
    \sum_{j=i_{c-1}+1}^{i_c} \alpha^{(j)} < 2(\Delta+1)
\end{equation}
Combining \Cref{eq:induction} and \Cref{eqn:5}, 
\begin{align*}
&\sum_{j=1}^{i_c} \alpha^{(j)} = \sum_{j=1}^{i_{c-1}} \alpha^{(j)} + \sum_{j=i_{c-1}+1}^{i_c} \alpha^{(j)} \\ &< 2(c-1)(\Delta+1) + 2(\Delta+1) = 2c(\Delta+1)
\end{align*}
\end{proof}

\begin{proof}{Proof of \Cref{thm:approx_2}}
We know that each matching in the distribution is group-fair because, in each iteration, the matching being computed in step \ref{4step:greedy} is a group-fair maximal matching.
Let $\X^{(i)}$ be the state of $\X$ after $i$ rounds of the loop in Algorithm \ref{alg4}. Let $M^{(i)}$ and $\alpha^{(i)}$ be the greedy maximal matching and it's coefficient being calculated in the $i^{th}$ round of the loop in Algorithm \ref{alg4}. Let Algorithm \ref{alg4} terminate after $k$ iterations, then $\X^{(k)} = \X - \displaystyle \sum_{\substack{(M^{(i)},\alpha^{(i)}) \in \mathcal{D}}}\alpha^{(i)}M^{(i)}$. Therefore,
\begin{align*}
    \X - \X^{(k)} = \displaystyle \sum_{\substack{(M^{(i)},\alpha^{(i)}) \in \mathcal{D}}}\alpha^{(i)}M^{(i)} \implies \frac{\X - \X^{(k)}}{\sum_{i=1}^k\alpha^{(i)}} \\
    = \displaystyle \sum_{\substack{(M^{(i)},\alpha^{(i)}) \in \mathcal{D}}}\frac{\alpha^{(i)}}{\sum_{i=1}^k\alpha^{(i)}}M^{(i)}
\end{align*}
In other words, $\frac{\X-\X^{(k)}}{\sum_{i = 1}^k\alpha^{(i)}}$ can be written as a convex combination of group-fair greedy maximal matchings. 

We will first find an upper bound for $\sum_{i = 1}^k\alpha^{(i)}$. 
Let $c'$ be such that $\frac{\norm{\X}_1}{2^{c'}}< \epsilon$, by \Cref{lem:alpha_bound},

$$\sum_{j=1}^{i_{c'}} \alpha^{(j)} \le 2c'(\Delta+1).$$ 
Since $\norm{x}_1 \leq n$, setting $\norm{x}_1 = n$, we have
$\frac{n}{2^{c'}} < \epsilon$. 
Setting $c' = \log (n/\epsilon) + 1$, we have $$\sum_{i=1}^{c'}\alpha^{(i)} < 2(\Delta + 1)(\log (n/\epsilon) + 1).$$
Let $f_{\epsilon} = 2(\Delta + 1)(\log (n/\epsilon) + 1)$, then $\frac{\X - \X^{(k)}}{f_{\epsilon}}$ can be written as a convex combination of group-fair integer matchings. Setting $\hat{x} = \frac{\X - \X^{(k)}}{f_{\epsilon}}$, $t = f_{\epsilon}$, and $\delta = \frac{\epsilon}{f_{\epsilon}}$ in \Cref{lem:ind_fair}, we have for each item $a \in A$ and subset $R_{a,k}$ $\forall$ $k \in [m]$, \begin{align*}
      &\frac{1}{f_{\epsilon}}\left(L_{a,k} - \epsilon\right) \leq \Pr_{\M \sim \cD}[\exists p\in R_{a,k} \textrm{ s.t. } (a,p)\in M] \\ &\leq  \frac{1}{f_{\epsilon}}\left(U_{a,k} + \epsilon\right).
\end{align*}

\noindent The
run time has been shown to be polynomial in \Cref{Alg4-run-time}. This proves the theorem.
\end{proof}

\noindent Now we prove \Cref{lem:ind_fair}, stated below. We use Lemma \ref{lem:ind_fair} to prove the individual fairness guarantees provided not just in \Cref{thm:approx_2} but also in the rest of the theorems.

\begin{restatable}{lemma}{indFair}\label{lem:ind_fair}
Let us consider a set of tuples, $\cD = \{(\M^{(i)}, \beta^{(i)})\}_{i \in [k]}$, where $\M^{(i)}$ is an integer matching and $\beta^{(i)}$ is a scalar, $\forall i \in [k]$, where $k \in \mathbb{Z}$. Let $\hat{x} = \sum_{i=1}^k \beta^{(i)}\M^{(i)}$ such that $\sum_{i=1}^k \beta^{(i)} = 1$, and $\norm{\hat{x} - \frac{\X}{t}}_1 \leq \delta$ where $\X$ is any feasible solution of \Cref{LP_disjoint}, $\delta \in [0,1)$, and $t \geq 1$. The probability that an item, $a \in A$, is matched to a platform $p \in S$, where $S \subseteq N(a)$, in a matching sampled from the support of $\mathcal{D}$ is $$\frac{L_{a,S}}{t} - \delta \leq \Pr_{\M \sim \cD}[\M \text{ matches } a \text{ to some platform in } S] \leq \frac{U_{a,S}}{t} + \delta.$$
\end{restatable}

\begin{proof}
    Given that $\norm{\Vec{\hat{x}} - \frac{\X}{t}}_1 \leq \delta$, therefore,
\begin{equation}\label{eq:mod1}
 \displaystyle\sum\limits_{(a,p) \in E}\left\lvert \left(\hat{x}_{ap} - \frac{x_{ap}}{t} \right)\right\rvert \leq \delta  
\end{equation}

Let's fix an arbitrary item $a \in A$, and let $S$ be an arbitrary subset of $N(a)$, then from \Cref{eq:mod1},
$$\displaystyle\sum\limits_{p \in S}\left\lvert \left(\hat{x}_{ap} - \frac{x_{ap}}{t} \right)\right\rvert \leq \delta \implies \left\lvert \displaystyle\sum\limits_{p \in S}\left(\hat{x}_{ap} - \frac{x_{ap}}{t} \right)\right\rvert \leq \delta$$
The last inequality holds due to triangle inequality. Therefore,
\begin{equation}\label{eq:mod3}
\frac{1}{t}\displaystyle\sum\limits_{p \in S}x_{ap} - \delta < \displaystyle\sum\limits_{p \in S}\hat{x}_{ap} < \frac{1}{t}\displaystyle\sum\limits_{p \in S}x_{ap} + \delta
\end{equation}

Since $\Vec{\hat{x}} = \sum_{i=1}^k \beta^{(i)}\Vec{M}^{(i)}$, the probability that an item $a \in A$ is matched to a platform $p \in S$, where $S \subseteq N(a)$, in a matching sampled from $\cD$ is
\begin{align*}
    \begin{split}
        & \Pr_{\Vec{M} \sim \cD}[\Vec{M} \text{ matches } a \text{ to a platform in } S] =  \displaystyle\sum_{\substack{i : M^{(i)} \text{ matches } \\ a \text{ to } p \in S}}  \beta^{(i)} \\ & = \displaystyle\sum\limits_{p \in S}\displaystyle\sum\limits_{\substack{i : M^{(i)} \text{ matches } \\ a \text{ to } p \in S}} \beta^{(i)} = \displaystyle\sum\limits_{p \in S}\hat{x}_{ap}
    \end{split}
\end{align*}

Therefore, from constraint \ref{LP-2} and \Cref{eq:mod3}, we have
\begin{align*}
    \begin{split}
        & \frac{L_{a,S}}{t} - \delta \leq \Pr_{\Vec{M} \sim \cD}[\Vec{M} \text{ matches } a \text{ to a platform in } S] \\& \leq \frac{U_{a,S}}{t} + \delta.
    \end{split}
\end{align*}
\end{proof}

\section{Algorithm for Disjoint Groups}\label{disjoint_groups}

In this section, we prove \Cref{thm:exact_algo} by working with an instance of a bipartite graph where each item belongs to exactly one group; that is, all the groups are disjoint. We first establish the fundamental module necessary for computing a probabilistic individually fair distribution over a set of integer group-fair matchings on a bipartite graph with disjoint groups.
To prove Theorems \ref{thm:approx_1_informal} and \ref{thm:approx_3_informal}, we first need to reduce the problem instance to one where $\Delta = 1$. Therefore, this module is key to proving Theorems \ref{thm:approx_3_informal}, and \ref{thm:approx_1_informal} as well. \Cref{thm:exact_algo} is restated below:

\exactAlgo*

\begin{lp}\label{LP_disjoint}
\begin{align}
&\max \sum_{(a,p) \in E} x_{ap} \\
\text{such that}\quad &L_{a,k} \leq \sum_{p \in R_{a,k}}\cx \leq U_{a,k},  & \forall a \in A, \forall k \in [m] \label{LP-3} \\  
&l_p \leq \sum_{a \in N(p)} \cx \leq u_p , &\forall p \in P &\label{LP-4}\\
&\cGL \leq \sum_{a \in A_h} \cx \leq \cGU , & \forall p \in P, \forall h \in [\chi] \label{LP-6}\\
&0 \leq \cx \leq 1 & \forall a \in A, \; \forall p\in P &\label{LP-8}
\end{align}
\end{lp}

\noindent We first give a sketch for the proof of \Cref{thm:exact_algo} using \Cref{LP_disjoint}, Algorithm \ref{alg:procedure} and $GFLP$ before providing a detailed proof of \Cref{thm:exact_algo}. Let us first look at the LP $GFLP$. 

\begin{definition}[\textbf{Group Fair Maximum Matching LP($GFLP$)}]\label{def:GFMMLP} 
This LP aims to find a maximum matching that does not violate any group or platform bounds. It is the same as the \Cref{LP_disjoint} without constraint \ref{LP-2}.
\end{definition}

\begin{observation}\label{ob:polytope}
Any feasible solution of \Cref{LP_disjoint} lies inside the polytope of $GFLP$(\Cref{def:GFMMLP}).
\end{observation}

\begin{lemma}[\cite{Rank-MaximalAndPOpular}]\label{lem:integrality}
Any vertex in the polytope of $GFLP$ is integral if $\forall p \in P$, $\forall h \in [\chi]$, the $l_p$, $u_p$, $\cGL$ and $\cGU$ values are integers.
\end{lemma}

\noindent The proof of \Cref{lem:integrality} is in the appendix. Let us first look at Algorithm \ref{alg:procedure}(Distribution-Calculator), which is an adaptation of the Birkhoff-von-Neumann algorithm(\cite{birkhoff1946}) to our setting. Birkhoff's Theorem states that the set of doubly stochastic matrices forms a convex polytope whose vertices are permutation matrices, and the Birkhoff-von-Neumann algorithm decomposes a bistochastic matrix into a convex combination of permutation matrices. 

\paragraph{Distribution-Calculator:} By \Cref{ob:polytope} and Carath\'eodory's theorem, any feasible solution of \Cref{LP_disjoint} can be written as a convex combination of extremal points in $GFLP$. Algorithm \ref{alg:procedure} takes an optimal solution of \Cref{LP_disjoint}, which maximizes the matching size and computes the above-mentioned convex combination over corner points of $GFLP$. Algorithm \ref{alg:procedure} first removes all edges $e \in E$ that have $x_e = 0$. The algorithm adjusts the upper and lower bounds to obtain an integer matching from $GFLP$. Specifically, it rounds up and down, respectively, the sum of $x_e$ values linked to edges in the vicinity of each platform or group-platform combination. Next, Algorithm \ref{alg:findAlpha}[Find-Coefficient] is used to compute an appropriate coefficient for the resulting integral matching. The coefficient should be such that after step \ref{proc:calc_x} of Algorithm \ref{alg:procedure}, the resulting point should lie within the polytope of $GFLP$ and either there is at least one edge $(a,p)$ such that $x_{ap}=0$ or at least one constraint becomes tight. We use this fact and induction to show that the integer matching being computed in step \ref{1step:MM} of Algorithm \ref{alg:procedure} is group-fair (\Cref{lem:loop_invariant}) and $\cD$ returned by Algorithm \ref{alg:procedure} is a distribution of said matchings (\Cref{lem:convexity}). This is also important to show that the algorithm terminates in polynomial time (\Cref{lem:runtime}). Finally, it is scaled to ensure that all the coefficients sum up to $1$. These steps are repeated until there are no edges left. 

\begin{algorithm}[t]
\caption{Distribution-Calculator$(\cI=(G,A_1 \cdots A_{\chi}, \Vec{l}, \Vec{u}, \Vec{L},\Vec{U}),\X,LP)$}\label{alg:procedure}
\nonl \textbf{Input} :  $\cI$, $\X$, $LP$ \\
\nonl \textbf{Output} : Distribution $\mathcal{D}$ over integer matchings \\
$G^{(0)} \leftarrow G, \X^{(0)} \leftarrow \X, \cD \leftarrow \phi, \alpha^{(0)} \leftarrow 0, \Gamma^{0} \leftarrow 1, \beta^{(0)} \leftarrow 1$ \\
\While{$\X^{(i)} \neq \mathbf{0}$}{
$i \leftarrow i+1, G^{(i)} \leftarrow G^{(i-1)} - \{\cE{a}{p} \; \lvert \; \x{i-1}_{ap} = 0\}$
\\$l^{(i)}_p \leftarrow \lfloor \sum_{a \in N(p)} \x{i-1}_{ap} \rfloor$, $u^{(i)}_p \leftarrow \lceil \sum_{a \in N(p)} \x{i-1}_{ap} \rceil$, $\forall p \in P$ \label{step:begin_mod}
\\$l^{(i)}_{p,\ch} \leftarrow \lfloor \sum_{a \in A_h} \x{i-1}_{ap} \rfloor, u^{(i)}_{p,\ch} \leftarrow \lceil \sum_{a \in A_h} \x{i-1}_{ap} \rceil$, $\forall p \in P$, $h \in [\chi]$ \label{step:end_mod}
\\$\M^{(i)} \leftarrow$ Matching returned by solving $LP$ on $G^{(i)}$ with $l^{(i)}_p, u^{(i)}_p, l^{(i)}_{p,\ch}, u^{(i)}_{p,\ch}$ as the bounds. \label{1step:MM}
\\ $\alpha^{(i)} \leftarrow$ \hyperref[alg:findAlpha]{Find-Coefficient}($G^{(i)}, A_1 \cdots A_{\chi}, \X^{(i-1)}, \M^{(i)}$) \label{proc:alpha}
\\$\X^{(i)}\leftarrow \frac{\X^{(i-1)}-\alpha^{(i)} \M^{(i)}}{1-\alpha^{(i)}}$ \label{proc:calc_x}
\\$\beta^{(i)} \leftarrow \Gamma^{(i-1)} \cdot  \alpha^{(i)}, \mathcal{D} \leftarrow \mathcal{D}\cup (\M^{(i)}, \beta^{(i)})$ \label{step:beta}
\\ $\Gamma^{(i)} \leftarrow \Gamma^{(i-1)} \cdot (1-\alpha^{(i)})$ \label{step:Gamma}
}
\If{$\cD == \phi$}{Return {\em `Infeasible'}}
Return $\cD$
\end{algorithm}

\begin{algorithm}[t]
\caption{Find-Coefficient($G', A_1 \cdots A_{\chi}, \X, \M$)} \label{alg:findAlpha}
\nonl \textbf{Input} : Graph $G'$, Groups $A_1 \cdots A_{\chi}$, $\X$, $\M$ \\
\nonl \textbf{Output} : Scalar $\alpha$ \\
$\alpha \leftarrow \min_{\cE{a}{p} \in \M}\cx$ \\
\For{$p \in P$}
{
\lIf{$\sum_{a \in N(p)} M_{ap} == \lceil \sum_{a \in N(p)} \cx \rceil$}
{
$temp \leftarrow \sum_{a \in N(p)} \cx - \lfloor \sum_{a \in N(p)} \cx \rfloor$ \label{step:alpha_is_lower}
}\lElse{
$temp \leftarrow \lceil \sum_{a \in N(p)} \cx \rceil - \sum_{a \in N(p)} \cx$ \label{step:alpha_is_upper}
}
\lIf{$temp < \alpha \text{ and } temp > 0$} 
{ \label{step:alpha_non_zero}
$\alpha \leftarrow temp$
}
\For{$h \in [\chi]$}
{
\lIf{$\sum_{a \in \cG} M_{ap} == \lceil \sum_{a \in \cG} \cx \rceil$}
{
$temp \leftarrow \sum_{a \in \cG} \cx - \lfloor \sum_{a \in \cG} \cx \rfloor$
}\lElse{
$temp \leftarrow \lceil \sum_{a \in \cG} \cx \rceil - \sum_{a \in \cG} \cx$
}
\lIf{$temp < \alpha \text{ and } temp > 0$}
{
$\alpha \leftarrow temp$
}
}
}
Return $\alpha$

\end{algorithm}

\begin{algorithm}[t]
\caption{Exact Algorithm$(\cI=(G,A_1 \cdots A_{\chi}, \Vec{l}, \Vec{u}, \Vec{L},\Vec{U}),\X))$}\label{alg1}
\nonl \textbf{Input} :  $\cI$ \\
\nonl \textbf{Output} : Distribution over matchings satisfying the guarantees in \Cref{thm:exact_algo}. \\
Solve \Cref{LP_disjoint} on $G$ with the parameters in the input instance, $\cI$, and store the result in $\X$ \label{step:lp_solve} \\
Return \hyperref[alg:procedure]{Distribution-Calculator}$(\cI, \X, GFLP)$
\end{algorithm}

In this section, we show that given an instance $\cI$ of our problem, an optimal solution of \Cref{LP_disjoint}, and the LP $GFLP$ as input, Algorithm \ref{alg:procedure} is a polynomial-time algorithm that returns a distribution over a set of group-fair matchings. Finally, by substituting $\hat{x} = \X$, $\delta = 0$, and $t = 1$ in Lemma \ref{lem:ind_fair}, we show that $\cD$ is a probabilistic individually fair distribution.
Given any feasible solution of \Cref{LP_disjoint}, say $\X$, we use Algorithm \ref{alg1} and $GFLP$ to compute a convex combination of integer matchings and prove that $\X$ can be written as the same.

\begin{lemma}\label{lem:loop_invariant}
$\X^{(i)}$ always lies within the polytope of $GFLP$, where $i+1$ denotes an arbitrary iteration of the while loop in Algorithm \ref{alg:procedure}.
\end{lemma}
\begin{proof}
We will prove this using induction. For the base case, $i+1=1$, $\X^{(0)} = \X$. Since $\X$ is an optimal solution of \Cref{LP_disjoint}, the Lemma holds by \Cref{ob:polytope}. Let us assume that the Lemma holds for $\X^{(i-1)}$ where $i$ denotes an arbitrary iteration of the while loop in Algorithm \ref{alg:procedure}. Now, we will show that the Lemma also holds for $\X^{(i)}$. If $\X^{(i-1)}$ is non zero, then there exists at least one $p \in P, h \in [\chi]$, such that $u_p^{(i)}$ and $\cGU^{(i)}$ values are at least one. Therefore, $\Vec{M}^{(i)}$ is a non empty matching on $G^{(i)}$,
since $GFLP$ returns a maximum matching that satisfies the updated group fairness constraints. 
First let us look at constraint \ref{LP-4} for an arbitrary platform, $p \in P$. Let $m^{(i)}_p$ be the number of edges picked in $\Vec{M}^{(i)}$ for platform $p$. From steps \ref{step:begin_mod} to \ref{step:end_mod} in Algorithm \ref{alg:procedure}, we know that we can have one of the following cases:
\begin{enumerate}
    \item $\displaystyle \sum_{a \in N(p)} x^{(i-1)}_{ap}$ is an integer in which case $l_p^{(i)} = u_p^{(i)} =\displaystyle \sum_{a \in N(p)} x^{(i-1)}_{ap}$. Since $\Vec{M}^{(i)}$ is an integer matching by \Cref{lem:integrality}, $m^{(i)}_p = l_p^{(i)} = u_p^{(i)} =\displaystyle \sum_{a \in N(p)} x^{(i-1)}_{ap}$. Therefore, for all values of $\alpha^{(i)} \in (0, 1]$, 
    $$\displaystyle \sum_{a \in N(p)}x^{(i)} = \frac{\sum_{a \in N(p)} x^{(i-1)}_{ap} -  \alpha^{(i)}\cdot m^{(i)}_p}{1-\alpha^{(i)}} = \displaystyle \sum_{a \in N(p)} x^{(i-1)}_{ap} = l_p^{(i)} = u_p^{(i)}$$
    
    \item $\displaystyle \sum_{a \in N(p)} x^{(i-1)}_{ap}$ is fractional in which case $u_p^{(i)} - l_p^{(i)} = 1$. Since $\Vec{M}^{(i)}$ is an integer matching by \Cref{lem:integrality}, $m^{(i)}_p = u_p^{(i)}$ or $m^{(i)}_p = l_p^{(i)}$. Therefore, we can have the following sub cases:
     \begin{enumerate}
         \item $m^{(i)}_p = l_p^{(i)}$: It is easy to see that for all values of $\alpha^{(i)} \in (0,1]$, the lower bound is always satisfied. Based on step \ref{step:alpha_is_upper} of the Routine \hyperref[alg:findAlpha]{Find-Coefficient} that is called in step \ref{proc:alpha} of Algorithm \ref{alg:procedure}, we know that $\alpha^{(i)} \leq \lceil \sum_{a \in N(p)} \cx^{(i-1)} \rceil - \sum_{a \in N(p)} \cx^{(i-1)} = u_p^{(i)} - \sum_{a \in N(p)} x^{(i-1)}_{ap}$. Therefore,
          \begin{align*}
             \begin{split}
                 & \sum_{a \in N(p)} x^{(i-1)}_{ap} \leq u_p^{(i)} -\alpha^{(i)}(u_p^{(i)} - l_p^{(i)}) \\ & \implies  \frac{\sum_{a \in N(p)} x^{(i-1)}_{ap} -  \alpha^{(i)}\cdot l^{(i)}_p}{1-\alpha^{(i)}} \leq u_p^{(i)}
             \end{split}
         \end{align*}
         Hence constraint \ref{LP-4} is not violated.
     
        \item $m^{(i)}_p = u_p^{(i)}$: It is easy to see that for all values of $\alpha^{(i)} \in (0,1]$, the upper bound is always satisfied. Based on step \ref{step:alpha_is_lower} of the Routine \hyperref[alg:findAlpha]{Find-Coefficient}, we know that $\alpha^{(i)} \leq \sum_{a \in N(p)} \cx^{(i-1)} - \lfloor \sum_{a \in N(p)} \cx^{(i-1)} \rfloor = \sum_{a \in N(p)} x^{(i-1)}_{ap} - l_p^{(i)}$. Following steps similar to the above sub case, we have $$l_p^{(i)} \leq \frac{\sum_{a \in N(p)} x^{(i-1)}_{ap} -  \alpha^{(i)}\cdot u^{(i)}_p}{1-\alpha^{(i)}} = \sum_{a \in N(p)}x^{(i)}$$
        Hence constraint \ref{LP-4} is not violated.
     \end{enumerate}
\end{enumerate}
Similar arguments can be used to show that constraint \ref{LP-6} is also not violated. Constraint \ref{LP-8} is satisfied trivially. Therefore, $\X^{(i)}$ also satisfies all the constraints of $GFLP$ and hence lies within the polytope of $GFLP$.
\end{proof}
 
\begin{claim}\label{claim:tight_bounds}
In an arbitrary iteration of the while loop in Algorithm \ref{alg:procedure}, say $i^{th}$ iteration, if there is an edge, $(a,p)$, such that $\sum_{a \in A_h}\x{i}_{ap} = u_{p,\ch}^{(i)}$ or $\sum_{a \in A_h}\x{i}_{ap} = l_{p,\ch}^{(i)}$ for some $h \in [\chi]$, then, $l_{p,\ch}^{(j)} = \sum_{a \in A_h}\x{j-1}_{ap} = u_{p,\ch}^{(j)}$, $\forall j \in \mathbb{Z}$ such that $j \geq i+1$. If Algorithm \ref{alg:procedure} completes in $k$ iterations, then $j \in \mathbb{Z}\cap [i+1, k]$. Similarly, for constraint \ref{LP-4}.
\end{claim} 

\begin{proof}
We can prove this by a simple induction on $j$ where $j \in \mathbb{Z}$ such that $j \geq i+1$. In this proof, we will make the implicit assumption that $j \leq k$ if Algorithm \ref{alg:procedure} completes in $k$ iterations. For the base case, $j=i+1$, $l_{p,\ch}^{(i+1)} = \lfloor \sum_{a \in A_h}\x{i}_{ap} \rfloor$, and $u_{p,\ch}^{(i+1)} = \lceil \sum_{a \in A_h}\x{i}_{ap} \rceil$. Since $\sum_{a \in A_h}\x{i}_{ap} = u_{p,\ch}^{(i)} = \lceil \sum_{a \in A_h}\x{i-1}_{ap} \rceil$ or $\sum_{a \in A_h}\x{i}_{ap} = l_{p,\ch}^{(i)} = \lfloor \sum_{a \in A_h}\x{i-1}_{ap} \rfloor$ by the assumption in the lemma, $\sum_{a \in A_h}\x{i}_{ap}$ is an integer, therefore,
$$l_{p,\ch}^{(i+1)} = \lfloor \sum_{a \in A_h}\x{i}_{ap} \rfloor = \lceil \sum_{a \in A_h}\x{i}_{ap} \rceil = u_{p,\ch}^{(i+1)}$$
For the induction step let's assume that for some arbitrary integer $j > i+1$, $\sum_{a \in A_h}\x{j}_{ap} = u_{p,\ch}^{(j)} = \lceil \sum_{a \in A_h}\x{j-1}_{ap} \rceil$, therefore, $\sum_{a \in A_h}\x{j}_{ap}$ is an integer by the induction hypothesis and hence
$$u_{p,\ch}^{(j+1)} = \lceil \sum_{a \in A_h}\x{j}_{ap} \rceil= \sum_{a \in A_h}\x{j}_{ap} = \lfloor \sum_{a \in A_h}\x{j}_{ap} \rfloor = l_{p,\ch}^{(j+1)}$$
\end{proof}
 
\begin{lemma}\label{lem:runtime} 
Algorithm \ref{alg1} terminates in polynomial time.
\end{lemma} 

\begin{proof}
We will show that in each iteration, at least an edge is removed, or at least one constraint becomes tight. From \Cref{claim:tight_bounds}, we know that once a constraint becomes tight, it stays so in the rest of the rounds. Therefore, we get a polynomial bound on the total number of iterations since the total number of constraints and edges is polynomial. Let us look at an arbitrary  iteration, say $i^{th}$ iteration. If $\alpha^{(i)} = \min_{e \in \Vec{M}^{(i)}}x^{(i-1)}_e$, then there is at least one edge, say $(a,p) \in E$, such that $x^{(i)}_{ap}=0$. Otherwise, there is some platform, say $p \in P$, such that $\alpha^{(i)} = \min \big(\sum_{a \in N(p)} \cx^{(i-1)} - l_p^{(i)}, u_p^{(i)} - \sum_{a \in N(p)} \cx^{(i-1)}\big)$, or there is some platform, say $p' \in P$, with some group, say $h \in [\chi]$, such that $\alpha^{(i)} = \min\big(\sum_{a \in \cG}x_{ap'}^{(i-1)} - l_{p',\cG}^{(i)}, u_{p',\cG}^{(i)} - \sum_{a \in \cG}x_{ap'}^{(i-1)}\big)$. Let $\alpha^{(i)} = \sum_{a \in N(p)} \cx^{(i-1)} - l_p^{(i)}$, then from step \ref{step:alpha_is_lower} of Routine \ref{alg:findAlpha} we know that $m_p^{(i)} = u_p^{(i)}$ where $m_p^{(i)}$ is the number of edges picked in $\Vec{M}^{(i)}$ for platform $p$. Therefore,
\begin{equation}\label{eq:alpha}
\begin{split}
    &  \sum_{a \in N(p)}x^{(i)} = \frac{\sum_{a \in N(p)} x^{(i-1)}_{ap} -  \alpha^{(i)}\cdot m^{(i)}_p}{1-\alpha^{(i)}} =  \\ & \frac{\sum_{a \in N(p)} x^{(i-1)}_{ap} -  (\sum_{a \in N(p)} \cx^{(i-1)} - l_p^{(i)})\cdot u_p^{(i)}}{1-(\sum_{a \in N(p)} \cx^{(i-1)} - l_p^{(i)})} 
\end{split} 
\end{equation}

Note that $\sum_{a \in N(p)} \cx^{(i-1)}$ must be fractional if $\alpha^{(i)} = \sum_{a \in N(p)} \cx^{(i-1)} - l_p^{(i)}$ based on step \ref{step:alpha_non_zero} of Routine \ref{alg:findAlpha}. Therefore, $l_p^{(i)} = \lfloor \sum_{a \in N(p)} \cx^{(i-1)} \rfloor = \lceil \sum_{a \in N(p)} \cx^{(i-1)} \rceil - 1 = u_p^{(i)} - 1$. Hence, from \Cref{eq:alpha}, $ \sum_{a \in N(p)}x^{(i)} =$
\begin{align*}
    \begin{split}
        & \frac{\sum_{a \in N(p)} x^{(i-1)}_{ap} -  (1 + \sum_{a \in N(p)} \cx^{(i-1)} - u_p^{(i)})\cdot u_p^{(i)}}{1-(1 + \sum_{a \in N(p)} \cx^{(i-1)} - u_p^{(i)})} \\& = \frac{(u_p^{(i)} - 1)(u_p^{(i)} - \sum_{a \in N(p)} \cx^{(i-1)})}{(u_p^{(i)} - \sum_{a \in N(p)} \cx^{(i-1)})} = l_p^{(i)}
    \end{split}
\end{align*}
Similarly, we can show that if 
\begin{enumerate}
    \item $\alpha^{(i)} = u_p^{(i)} - \sum_{a \in N(p)} \cx^{(i-1)}$, then $\sum_{a \in N(p)}\cx^{(i)} = u_p^{(i)}$.
    \item $\alpha^{(i)} = \sum_{a \in \cG}x_{ap'}^{(i-1)} - l_{p',\cG}^{(i)}$, then $\sum_{a \in \cG}x_{ap'}^{(i)} = l_{p',\cG}^{(i)}$.
    \item $\alpha^{(i)} = u_{p',\cG}^{(i)} - \sum_{a \in \cG}x_{ap'}^{(i-1)} $, then $\sum_{a \in \cG}x_{ap'}^{(i)} = u_{p',\cG}^{(i)}$.
\end{enumerate}
Therefore, if $\alpha^{(i)} \neq \min_{e \in \Vec{M}^{(i)}}x^{(i-1)}_e$, that is if an edge is not removed, then either the left inequality or right inequality of constraint \ref{LP-4} becomes tight in the $i^{th}$ round for some platform or there is some platform, say $p' \in P$, such that either left inequality or right inequality of constraint \ref{LP-6} becomes tight in the $i^{th}$ round for some group, say $h \in [\chi]$. 
Since the total number of constraints is $O(|V|)$, $|E| = O(|V|^2)$, and the routine \hyperref[alg:findAlpha]{Find-Coefficient} runs in $O(|V|)$ time, we get $O(|V|^3)$ iterations for the loop in Algorithm \ref{alg:procedure}.
Since $GFLP$ can be solved in polynomial time, Algorithm \hyperref[alg:procedure]{Distribution-Calculator} is a polynomial-time algorithm. 
\Cref{LP_disjoint} can be solved in polynomial time, therefore, Algorithm \ref{alg1} also runs in polynomial time.
\end{proof}
 
\begin{claim}\label{clm:x0induction}
Let Algorithm \ref{alg1} terminate in $k$ rounds and return a set of tuples, $\cD = \{(\Vec{M}^{(i)}, \beta^{(i)})\}_{i \in [k]}$, then, $\forall i \in [k]$,
$$\X^{(i)} = \frac{\X^{(0)} - \sum_{j=1}^i\beta^{(j)}\Vec{M}^{(j)}}{\Gamma^{(i)}}$$
\end{claim}

\begin{proof}
From steps \ref{step:beta} and \ref{step:Gamma} in Algorithm \ref{alg1}, we know that $\beta^{(i)} = \Gamma^{(i-1)}\alpha^{(i)}$, and $\Gamma^{(i)} = \Gamma^{(i-1)}(1-\alpha^{(i)})$ respectively, and $\Gamma^{(0)} = 1$. For the base case, $i=1$, we know from step \ref{proc:calc_x} in Algorithm \ref{alg:procedure} that $\X^{(1)} = \frac{\X^{(0)} - \alpha^{(1)}\Vec{M}^{(1)}}{1-\alpha^{(1)}}$. It is easy to see that $\Gamma^{(1)} = (1-\alpha^{(1)})$ and $\beta^{(1)} = \alpha^{(1)}$, therefore,
$$\X^{(1)} = \frac{\X^{(0)} - \beta^{(1)}\Vec{M}^{(1)}}{\Gamma^{(1)}}$$
For the induction step, for some $i \in \mathbb{Z}\cap(1,k]$, let $\X^{(i-1)} = \frac{\X^{(0)} - \sum_{j=1}^{i-1}\beta^{(j)}\Vec{M}^{(j)}}{\Gamma^{(i-1)}}$. We know that $\X^{(i)} = \frac{\X^{(i-1)} - \alpha^{(i)}\Vec{M}^{(i)}}{1-\alpha^{(i)}}$, therefore, by induction hypothesis,
$$\X^{(i)} = \frac{\frac{\X^{(0)} - \sum_{j=1}^{i-1}\beta^{(j)}\Vec{M}^{(j)}}{\Gamma^{(i-1)}} - \alpha^{(i)}\Vec{M}^{(i)}}{1-\alpha^{(i)}} = \frac{\X^{(0)} - \sum_{j=1}^{i-1}\beta^{(j)}\Vec{M}^{(j)} - \Gamma^{(i-1)}\alpha^{(i)}\Vec{M}^{(i)}}{\Gamma^{(i-1)}(1-\alpha^{(i)})}$$
hence,
$$\X^{(i)} = \frac{\X^{(0)} - \sum_{j=1}^{i}\beta^{(j)}\Vec{M}^{(j)}}{\Gamma^{(i)}}$$
\end{proof}

\begin{lemma}\label{lem:convexity}
Let Algorithm \ref{alg1} terminate in $k$ rounds and return a set of tuples, $\cD = \{(\Vec{M}^{(i)}, \beta^{(i)})\}_{i \in [k]}$, then, $\X = \displaystyle \sum_{i=1}^k \beta^{(i)}\Vec{M}^{(i)}$, where $\X$ is computed in step \ref{step:lp_solve} of Algorithm \ref{alg1} and $\displaystyle \sum_{i=1}^k \beta^{(i)} = 1$.
\end{lemma} 

\begin{proof}
From \Cref{clm:x0induction}, we know that $\forall i \in [k]$,
$$\X^{(k)} = \frac{\X^{(0)} - \sum_{i=1}^k\beta^{(i)}\Vec{M}^{(i)}}{\Gamma^{(k)}}.$$ 
Since $\X^{(k)} = \Vec{0}$, we have $\X^{(0)} = \X = \displaystyle \sum_{i=1}^k \beta^{(i)}\Vec{M}^{(i)}$.

\noindent We will prove that $\displaystyle \sum_{i=1}^k \beta^{(i)} = 1$, using induction on $i$, backwards from $k$ to $0$. 
For the base case, $i=k$, $\X^{(k)} = \Vec{0}$, therefore, for any real values of $\alpha$, $\X^{(k)} = (1-\alpha)\Vec{M} + \alpha\Vec{M}$, where $\Vec{M}$ is an empty matching. Note that if $u_p, l_p, u_{p,\ch}, l_{p,\ch}$ are set to $0$, $\forall p \in P, h \in [\chi]$, $GFLP$ would compute an empty matching.
Therefore, $\X^{(k)}$ can be written as a convex combination of integer matchings computed by $GFLP$.
For the induction step, let us assume that $\X^{(i+1)} = \sum_{j}\gamma_j \Vec{M}_j$, where $\Vec{M}_j$ is an integer matching computed by $GFLP$ for all values of $j$ and  $\sum_{j}\gamma_j=1$, for some $i \in [k-1]$. We know that $\X^{(i+1)} = \frac{\X^{(i)}- \alpha^{(i)}\Vec{M}^{(i)}}{1-\alpha^{(i)}}$. Therefore,
$$\X^{(i)} = (1-\alpha^{(i)})\X^{(i+1)} + \alpha^{(i)}\Vec{M}^{(i)} = (1-\alpha^{(i)})\sum_{j}\gamma_j \Vec{M}_j + \alpha^{(i)}\Vec{M}^{(i)}$$
Since $\sum_{j}\gamma_j=1$, by the induction hypothesis, $(1-\alpha^{(i)})\sum_{j}\gamma_j + \alpha^{(i)} = 1$, therefore, $\X^{(i)}$ is also a convex combination of integer matchings computed by $GFLP$.

\noindent From \Cref{clm:x0induction} we know that, $\forall i \in [k]$, $\X^{(i)} = \frac{\X^{(0)} - \sum_{j=1}^i\beta^{(j)}\Vec{M}^{(j)}}{\Gamma^{(i)}}$, hence,
$$\X^{(0)} = \Gamma^{(i)}\X^{(i)} + \sum_{j=1}^i\beta^{(j)}\Vec{M}^{(j)}.$$
Since, we have already shown that, $\X^{(i)}$ is a convex combination of integer matchings computed by $GFLP$ using induction, we just need to show that $\Gamma^{(i)} + \sum_{j=1}^i\beta^{(j)}= 1$. Expanding $\Gamma^{(i)}$ and $\sum_{j=1}^i\beta^{(j)}$, we have
$$\Gamma^{(i)} + \sum_{j=1}^i\beta^{(j)} = \Pi_{j=1}^i(1-\alpha^{(j)}) + \sum_{j=1}^i \Pi_{l=1}^{j-1}(1 - \alpha^{(l)})\alpha^{(j)} = 1$$
Therefore, $\displaystyle \sum_{i=1}^k \beta^{(i)} = 1$.
\end{proof}

\begin{proof}[Proof of \Cref{thm:exact_algo}]
Let Algorithm \ref{alg1} terminate in $k$ rounds and return a set of tuples, $\cD = \{(\Vec{M}^{(i)}, \beta^{(i)})\}_{i \in [k]}$. Therefore, $\X = \sum_{i=1}^k  \beta^{(i)}\Vec{M}^{(i)}$ and $\sum_{i=1}^k \beta^{(i)} = 1$  by \Cref{lem:convexity}, where $\X$ is an optimal solution of \Cref{LP_disjoint}. We know that after every iteration, we get another point inside the polytope of $GFLP$ by \Cref{lem:loop_invariant}, therefore, in every iteration, the integer matching being computed in step \ref{1step:MM} of Algorithm \ref{alg:procedure} satisfies group fairness constraints. 
Therefore, Algorithm \ref{alg1} returns a distribution over group-fair integer matchings.
By substituting $\Vec{\hat{x}} = \X$, $\delta = 0$, and $t = 1$ in \Cref{lem:ind_fair}, we get that the probability that an item $a \in A$ is matched to a platform $p \in S$, where $S \subseteq N(a)$, in a matching sampled from $\cD$ is $L_{a,S} \leq \sum\limits_{p \in S}\cx \leq U_{a,S}$, $\forall a \in A, S \subseteq N(a)$. Hence, $\cD$ is a probabilistic individually fair distribution. The run time has been shown to be polynomial in \Cref{lem:runtime}. This proves the theorem.
\end{proof}
\section{$O(g)$ Bicriteria Approximation Algorithms}\label{sec:g-approx}

In this section, we work with an instance of a bipartite graph, $G(A,P,E)$, where the items in the neighborhood of any platform $p$ belong to at most $g$ (\Cref{def:total_plat_grps}) distinct groups, and any item, $a \in A$, can belong to at most $\Delta$ groups. We first reduce this instance to one where $\Delta = 1$, then use $GFLP$ with specific bounds in the form of \Cref{LP_mod}, and Algorithm \ref{alg:procedure} to compute a distribution over matchings in Algorithm \ref{alg2}. Since \Cref{disjoint_groups} also addresses an instance where the groups are disjoint, we will use Lemmas from \Cref{disjoint_groups} and an analysis technique similar to \Cref{disjoint_groups} to prove \Cref{thm:approx_1_informal}, formally stated below.

\begin{theorem}[Formal version of \Cref{thm:approx_1_informal}]\label{thm:approx_1}
Given an instance of our problem where each item belongs to at most $\Delta$ groups, and $\cGU \ge g$ $\forall p \in P, h \in [\chi]$, there is a polynomial-time algorithm that computes a distribution $\cD$ over a set of group-fair matchings such that the expected size of a matching sampled from $\cD$ is at least $\frac{OPT}{2g}$. Given the individual fairness parameters, $L_{a,k},U_{a,k} \in [0,1]$, for each item $a \in A$ and subset $R_{a,k}$ $\forall$ $k \in [m]$, \begin{align*}
      \frac{L_{a,k}}{2g} \leq \Pr_{\M \sim \cD}[\exists p\in R_{a,k} \textrm{ s.t. } (a,p)\in M] \leq  \frac{U_{a,k}}{2g}.
\end{align*}
\end{theorem}

\noindent Let us first define $g$ formally and then look at \Cref{LP_mod}. To formally define $g$, we first need the following definition:
\begin{definition}\label{def:platform_groups}
    $\boldsymbol{C_p} = \{C_{p,h}: C_{p,h} \neq \phi\}_{h \in [\chi]}$ denotes a set of groups for any platform $p$ such that $C_{p,h} = A_h \cap N(p)$, for some $h \in [\chi]$. Here $N(p)$ denotes the set of neighbors of $p$ in $G$.
\end{definition}

\begin{definition}\label{def:total_plat_grps}
    $\boldsymbol{g} = \max_{p \in P}|C_p|$.
\end{definition}

\begin{lp}\label{LP_mod}
\begin{align}
&\max \sum_{(a,p) \in E} x_{ap} &\label{LP-12}\\
\text{such that}\quad &\sum_{a \in A_h} x_{ap} \leq \left\lfloor\frac{\cGU}{g}\right\rfloor \;, &\forall  h \in [\chi] ,\; \forall p \in P &\label{LP-13}\\
& 0 \leq x_{ap} \leq 1 &\forall (a,p) \in E &\label{LP-14}
\end{align}
\end{lp}

\begin{algorithm}[t]
\caption{$2g$-BicriteriaApprox$(\cI=(G,A_1 \cdots A_{\chi}, \Vec{l}, \Vec{u}, \Vec{L},\Vec{U}))$}
\label{alg2}
\nonl \textbf{Input} : $\cI$ \\
\nonl \textbf{Output} : Distribution over matchings satisfying the guarantees in \Cref{thm:approx_3}. \\
Solve \Cref{LP_Primal_IF} on $G$ with the parameters in the input instance, $\cI$, and store the result in $\X$ \\
$g = \max_{p \in P}|C_p|$ (\Cref{def:platform_groups} and \ref{def:total_plat_grps})\\
For each item $a \in A$, we remove it from every group other than $C_a$ where $C_a = \argmin_{C \in C_p:a \in C}\cGU$. Let the resulting graph be $G'$. \label{3step:disjoint}\\ 
$\cI'=(G',A'_1 \cdots A'_{\chi}, \Vec{l}, \Vec{u}, \Vec{L},\Vec{U})$\\
Return \hyperref[alg:procedure]{Distribution-Calculator}$(\cI', \frac{\X}{2g}, \Cref{LP_group_approx})$ 
\end{algorithm}

\begin{observation}\label{ob:polytope_2g}
Let $\X$ be a feasible solution of \Cref{LP_Primal_IF}, and $\cGU \ge g$ $\forall p \in P, h \in [\chi]$, then $\frac{\X}{2g}$ lies inside the polytope of \Cref{LP_mod}.
\end{observation}
\begin{proof}
Let $d$ be any positive real number, then we will consider the following two cases:
\begin{enumerate}
    \item $d \geq 2$: It is trivial to see that $\frac{d}{2} \leq d-1 \leq \lfloor d \rfloor$ in this case.
    \item $d \in [1,2)$: In this case, $d = 1+\delta$ where $\delta \in [0,1)$. Therefore, $\frac{\delta}{2} < \frac{1}{2}$, hence,
    $$\frac{d}{2} = \frac{1}{2} + \frac{\delta}{2} < 1 = \lfloor d \rfloor$$
\end{enumerate}
Therefore, for any positive real number $d \geq 1$,
$$\frac{d}{2} \leq \lfloor d \rfloor$$ 
Since $\X$ is a feasible solution of \Cref{LP_Primal_IF}, $\sum_{a \in A_h} x_{ap} \leq \cGU$ $\forall p \in P, h \in [\chi]$ by constraint \ref{LP-20}, therefore, $\forall p \in P, h \in [\chi]$
$$\frac{\sum_{a \in A_h} x_{ap}}{2g} \leq \frac{\cGU}{2g} \leq \left\lfloor\frac{\cGU}{g}\right\rfloor$$
The last inequality holds because of the assumption $\cGU \ge g$ $\forall p \in P, h \in [\chi]$, which implies $\frac{\cGU}{g} \geq 1$ $\forall p \in P, h \in [\chi]$. Therefore $\frac{\X}{2g}$ satisfies constraint \ref{LP-13}, constraint \ref{LP-14} is also satisfied because $\forall (a,p) \in E$, $0 \le x_{ap} \le 1$, therefore, $0 \le \frac{x_{ap}}{2g} \le 1$.
\end{proof}

\begin{lemma}\label{lem:2g-grp_fair}
Let $\X$ be any optimal solution of LP \ref{LP_mod} on the graph resulting after step \ref{3step:disjoint} in Algorithm \ref{alg2}, say $G'$, then $\X$ is an integer matching on $G'$ that satisfies the group fairness constraint \ref{LP-10}.
\end{lemma}

\begin{proof}
Any vertex solution of LP \ref{LP_mod} on $G'$ is integral if the groups are disjoint, by \Cref{lem:integrality}, therefore, $\X$ is an integer matching on $G'$. Let us fix an arbitrary platform, $p$, and let $T_p$ denote a set of groups such that $A_h \cap N(p) \neq \phi$, $\forall \ch \in T_p$. Let us number all the groups in $T_p$ in the ascending order of their upper bounds, breaking ties arbitrarily, that is, for any two groups, say $\ch_i,\ch_j \in T_p$, $u_{p,\ch_j} > u_{p,\ch_i}$ iff $j > i$. Let us consider an arbitrary group, $\ch_q \in T_p$, with upper bound $u_{p,\ch_q}$. Any item $a \in A$ that has been removed from this group in step \ref{3step:disjoint} of Algorithm \ref{alg2} could only be in one of the groups from $\ch_1$ to $\ch_{q-1}$. This is because an item stays in the group with the lowest upper bound. 
Let $m_i = \displaystyle\sum_{a \in A_{\ch_i}}x_{ap}$, then $m_i \leq \left\lfloor\frac{u_{p,\ch_i}}{g}\right\rfloor$ due to constraint \ref{LP-13}. Therefore,
\begin{align*}
    &\sum_{i=1}^q m_i \le \sum_{i=1}^q\left\lfloor\frac{u_{p,\ch_i}}{g} \right\rfloor \le \sum_{i=1}^q\frac{u_{p,\ch_i}}{g} \le \sum_{i=1}^q\frac{u_{p,\ch_q}}{g} \le \\ &g\cdot\frac{u_{p,\ch_q}}{g} = u_{p,\ch_q}
\end{align*}
Therefore, $\displaystyle\sum_{a \in A_h}x_{ap} \le \cGU$ $\forall h \in T_p$ after all the items are returned to all the groups they belonged to in the original graph. Therefore, $\X$ satisfies constraint \ref{LP-10}.
\end{proof}

\begin{lemma}\label{lem:2g_approx_runtime}
Given a bipartite graph $G(A, P, E)$ with possibly non-disjoint groups and an optimal solution of \Cref{LP_Primal_IF}, Algorithm \ref{alg2} returns a distribution over integer matchings in polynomial time, such that each matching satisfies group fairness constraints.
\end{lemma}
\begin{proof}
We start with an optimal solution of \Cref{LP_Primal_IF}, therefore, $\frac{\X}{2g}$ is a feasible solution of \Cref{LP_mod} by \Cref{ob:polytope_2g}. Let $\X^{(i)}$ be the state of the optimal solution of \Cref{LP_Primal}, $\X$, after the $i^{th}$ iteration of Algorithm \ref{alg:procedure}. Note that \Cref{LP_mod} is $GFLP$(\Cref{def:GFMMLP}) with specific upper and lower bounds, therefore, $\X^{(i)}$ always lies within the polytope of \Cref{LP_mod} by \Cref{lem:loop_invariant},$\forall i \in [k-1]$, where $k$ is the number of iterations after which Algorithm \ref{alg:procedure} terminates. Therefore, if $x^{(i-1)}$ is non empty, a non empty integer matching is computed in step \ref{1step:MM} of Algorithm \ref{alg:procedure} for $k$ rounds and by \Cref{lem:2g-grp_fair} we know that each such matching satisfies group fairness constraints. From \Cref{lem:convexity}, we know that $\X$ can be written as a convex combination of integer matchings computed by \Cref{LP_mod}. Therefore, Algorithm \ref{alg2} returns a distribution over group-fair integer matchings. The run time of Algorithm \ref{alg:procedure} has been shown to be polynomial in \Cref{lem:runtime}, since \Cref{LP_Primal_IF} can be solved in polynomial time, Algorithm \ref{alg2} also runs in polynomial time.
\end{proof}

\begin{proof}[Proof of \Cref{thm:approx_1}]
Let $\X$ be any optimal solution of \Cref{LP_Primal_IF}, then, Algorithm \ref{alg2} can be used to represent $\frac{\X}{2g}$ as a distribution, say $\cD$, of integer group-fair matchings in polynomial time by \Cref{lem:2g_approx_runtime}. By setting $\hat{x} = \X$, $\delta = 0$, and $t = 2g$ in \Cref{lem:ind_fair}, we have for each item $a \in A$ and subset $R_{a,k}$ $\forall$ $k \in [m]$, \begin{align*}
      \frac{L_{a,k}}{2g} \leq \Pr_{\M \sim \cD}[\exists p\in R_{a,k} \textrm{ s.t. } (a,p)\in M] \leq  \frac{U_{a,k}}{2g}.
\end{align*}
This proves the theorem.
\end{proof}

\subsection{Group Fairness Violation}\label{non_disjoint_groups}
In this section, we formally state and prove \Cref{thm:approx_3_informal}. We use \Cref{LP_group_approx} instead of \Cref{LP_mod} to reduce the problem to something similar to the problem we saw in \Cref{disjoint_groups}, then use Algorithm \ref{alg3}, which is a slightly modified version of Algorithm \ref{alg2}, and Lemmas from \Cref{disjoint_groups} to prove \Cref{thm:approx_3_informal} which is formally stated below. 

\begin{theorem}[Formal version of \Cref{thm:approx_3_informal}]\label{thm:approx_3}
Given an instance of our problem with no lower bound constraints where each item belongs to at most $\Delta$ groups, and $\cGU \ge g$ $\forall p \in P, h \in [\chi]$, we provide a polynomial-time algorithm that computes a distribution ${\cD}$ over a set of matching. The expected size of a matching sampled from ${\cD}$ is at least $\frac{OPT}{g}$, and each matching in the distribution violates group fairness by an additive factor of at most $\Delta$. Given the individual fairness parameters $L_{\mathrm{a},S} \in [0,1]$ and $U_{\mathrm{a},S} \in [0,1]$ for each item $\mathrm{a} \in A$ and each subset $S \subseteq N(\mathrm{a})$,
$$\frac{L_{\mathrm{a},S}}{g} \leq \Pr_{\M \sim \cD}[\M \text{ matches } \mathrm{a} \text{ to a platform in } S] \leq \frac{U_{\mathrm{a},S}}{g}$$
\end{theorem}

\begin{lp}\label{LP_group_approx}
\begin{align}
&\max \sum_{(a,p) \in E} x_{ap} &\label{LP-15}\\
\text{such that}\quad &\sum_{a \in C} x_{ap} \leq \left\lceil\frac{\cGU}{g}\right\rceil \;, &\forall  h \in [\chi] ,\; \forall p \in P  &\label{LP-16}\\
& 0 \leq x_{ap} \leq 1 &\forall (a,p) \in E &\label{LP-17}
\end{align}
\end{lp}

\begin{observation}\label{ob:polytope_new}
Let $\X$ be a feasible solution of \Cref{LP_Primal}, then $\frac{\X}{g}$ lies inside the polytope of \Cref{LP_group_approx}.
\end{observation}

\begin{algorithm}[t]
\label{alg3}
\caption{$g$-BicriteriaApprox$(\cI=(G,A_1 \cdots A_{\chi}, \Vec{l}, \Vec{u}, \Vec{L},\Vec{U}))$}
\nonl \textbf{Input} : $\cI$ \\
\nonl \textbf{Output} : Distribution over matchings satisfying the guarantees in \Cref{thm:approx_3}. \\
Solve \Cref{LP_Primal} augmented with constraint \ref{LP-2} on $G$ and store the result in $\X$ \\
$g = \max_{p \in P}|C_p|$ (\Cref{def:total_plat_grps})\\
For each item $a \in A$, we remove it from every group other than $C_a$ where $C_a = \argmin_{C \in C_p:a \in C}\cGU$. Let the resulting graph be $G'$. \label{2step:disjoint}\\ 
$\cI'=(G',A'_1 \cdots A'_{\chi}, \Vec{l}, \Vec{u}, \Vec{L},\Vec{U})$\\
Return \hyperref[alg:procedure]{Distribution-Calculator}$(\cI', \frac{\X}{g}, \Cref{LP_group_approx})$ 
\end{algorithm}

\begin{lemma}\label{lem:grp_fair}
The solution computed by \Cref{LP_group_approx} in Algorithm \ref{alg3} is an integer matching that violates the group fairness constraint \ref{LP-10} by an additive factor of at most $\Delta$.
\end{lemma}

\begin{proof}
Any optimal solution of LP \ref{LP_group_approx} on $G'$ is integral if the groups are disjoint, by \Cref{lem:integrality}, therefore, $\X$ is an integer matching on $G'$. Let's fix an arbitrary platform, $p$, and let $T_p$ denote a set of groups such that $A_h \cap N(p) \neq \phi$, $\forall \ch \in T_p$. Let us number all the groups in $T_p$ in the ascending order of their upper bounds, breaking ties arbitrarily, that is, for any two groups say $\ch_i,\ch_j \in T_p$, $u_{p,\ch_j} > u_{p,\ch_i}$ iff $j > i$. Let us consider an arbitrary group, $\ch_q \in [\chi]$, with upper bound $u_{p,\ch_q}$.  Any item $a \in A$ that has been removed from this group in step \ref{2step:disjoint} of Algorithm \ref{alg3} could only be in one of the groups from $\ch_1$ to $\ch_{q-1}$. This is because an item stays in the group with the lowest upper bound.  Let $m_i = \displaystyle\sum_{a \in \ch_i}x_{ap}$, then $m_i \leq \left\lceil\frac{u_{p, \cG}}{g}\right\rceil$ due to constraint \ref{LP-16}. Therefore,

$$\sum_{i=1}^q m_i \le \sum_{i=1}^q \left\lceil\frac{u_{p,\ch_i}}{g} \right\rceil \le \sum_{i=1}^q\left(\frac{u_{p,\ch_i}}{g} + 1 \right) \le \sum_{i=1}^q\left(\frac{u_{p,\ch_q}}{g} + 1\right) \le \Delta\cdot\frac{u_{p,\ch_q}}{g} + \Delta \le u_{p,\ch_q} + \Delta$$
The second last inequality holds because any item can belong to at most $\Delta$ groups. Let $x_{ap} \in \{0,1\}$ be the value assigned to an edge $\cE{a}{p} \in E$ in $\M$ returned by LP \ref{LP_group_approx}, then $\displaystyle\sum_{a \in A_h}x_{ap} \le \cGU + \Delta$ $\forall h \in [\chi]$ after all the items are returned to all the groups they belonged to in the original graph. Therefore, $\M$ violates constraint \ref{LP-10} by an additive factor of at most $\Delta$.
\end{proof}

\begin{lemma}\label{lem:runtime_new}
Given a bipartite graph $G(A, P, E)$ with possibly non-disjoint groups and an optimal solution of \Cref{LP_Primal}, Algorithm \ref{alg3} returns a distribution over integer matchings such that each matching violates group fairness constraints by an additive factor of at most $\Delta$, in polynomial time.
\end{lemma}
\begin{proof}
The proof is similar to the proof of \Cref{lem:2g_approx_runtime} with one key difference that in each iteration, the matching being computed in step \ref{1step:MM} of Algorithm \ref{alg:procedure} does not satisfy group fairness constraints but violates group fairness constraints by an additive factor of at most $\Delta$ by \Cref{lem:grp_fair}.
\end{proof}

\begin{proof}[Proof of \Cref{thm:approx_1}]
Let $\X$ be any optimal solution of \Cref{LP_Primal} augmented with \ref{LP-2}, then, Algorithm \ref{alg3} can be used to represent $\frac{\X}{g}$ as a convex combination of integer matchings that violate group fairness constraints by an additive factor of at most $\Delta$, in polynomial time by \Cref{lem:runtime_new}. By setting $\hat{x} = \X$, $\delta = 0$, and $t = g$ in \Cref{lem:ind_fair}, we have $\forall a \in A, S \subseteq N(a)$,
$$\frac{L_{a,S}}{g} \leq \Pr_{\M \sim \cD}[\M \text{ matches } a \text{ to some platform in } S] \leq \frac{U_{a,S}}{g}$$
The run time of Algorithm \ref{alg:procedure} has been shown to be polynomial in \Cref{lem:runtime}, since \Cref{LP_Primal} can be solved in polynomial time, Algorithm \ref{alg3} also runs in polynomial time. This proves the theorem.
\end{proof}

\section{Experiments}\label{experiments}
In this section, we apply our main contribution, approximation algorithm \ref{alg4} from \Cref{thm:approx_2}, on two real-world datasets. The runtime bottleneck of our primary solution (Algorithm \ref{alg4}) is the execution time of \Cref{LP_Primal} (appendix). \Cref{LP_Primal} is a simplified version of \Cref{LP_disjoint} with polynomial number of variables and constraints and Algorithm \ref{alg4} solves \Cref{LP_Primal} exactly once. Therefore, this solution is scalable whenever a practical LP solver is used. In our experiments on standard datasets, the algorithm performs much better than the $2(\Delta + 1)(\log (n/\epsilon)+1)$ approximation guarantee provided by \Cref{thm:approx_2}. Here $n$ is the total number of items, $\Delta$ is the maximum number of groups an item can belong to, and $\epsilon > 0$ is a small value.  There are no comparison experiments since there are no benchmarks for solving this exact problem. We use experiments to validate and demonstrate the practical efficiency of Algorithm \ref{alg4}.
    
\subsection{Datasets}\label{datasets}
\paragraph{Employee Access data}\footnote{https://www.kaggle.com/datasets/lucamassaron/amazon-employee-access-challenge}: This data is from Amazon, collected from 2010-2011, and published on the Kaggle platform. We use the testing set with 58921 samples for our experiments. Each row in the dataset represents an access request made by an employee for some resource within the company. In our model, the employees and the resources correspond to items and platforms, respectively, and each request represents an edge. We group the employees based on their role family. An employee can make multiple requests, each under a different role family. Therefore, each item can have edges to different platforms and belong to more than one group. 
We run our experiments on datasets of sizes 1000, 2000, 3000, and 5000 sampled from this dataset.
    
\paragraph{Grant Application Data}\footnote{https://www.kaggle.com/competitions/unimelb/data}: This data is from the University of Melbourne on grant applications collected between 2004 and 2008 and published on the Kaggle platform. We use the training set with 8,707 grant applications for our experiments. In our model, the applicants and the grants correspond to items and platforms, respectively, and each grant application represents an edge. We group the applicants based on their research fields. The same applicant, de-identified in the dataset, can apply to different grants under different research fields represented as RFCD code in the dataset. Therefore, each item can have edges to different platforms and belong to more than one group.
    
\subsection{Experimental Setup and Results}\label{exp_setup_results}
We implement our algorithm in Python 3.7 using the libraries NumPy, scipy, and Pandas. All the experiments we run using Google colab notebook on a virtual machine with Intel(R) Xeon(R) CPU $@$ $2.20$GHz and $13$GB RAM. Both the datasets on which we run our algorithm, are taken from Kaggle. We run our experiments on one complete dataset and three different sample sizes on another dataset. The sample size denotes the total number of rows present in the unprocessed sample. The total number of edges can differ from the sample size after data cleaning like removing null values and dropping duplicate edges if any. For group fairness bounds, we set the same upper and lower bounds for each platform group pair. If $n$ is the number of items, $m$ is the number of platforms, and $g$ is the number of groups, the upper bound is $\frac{kn}{mg}$, where $k = \lceil\frac{mg}{n} \rceil$. All the lower bounds are set to $0$. For individual fairness constraints, we first choose a random permutation of the platforms to create a ranking and then add constraints such that an item should have $\frac{r}{2}\%$ chance of being matched to a platform in the top $r\%$ in the ranking. For all the runs, $\epsilon = 0.0001$.

We use the solution obtained by solving \Cref{LP_disjoint} as an upper bound on $OPT$. We denote it by $UB$, and $SOL$ denotes the expected size of the solution given by Algorithm \ref{alg4} on different samples. Let $2(\Delta + 1)(\log (n/\epsilon) + 1)$ be denoted by `approx'. In Table \ref{table1}, we compare the actual approximation ratio, $UB/SOL$, with the theoretical approximation ratio, `approx'. As can be seen in \Cref{table1}, in our experiments on standard real datasets, the algorithm performs much better than the worst-case theoretical guarantee provided by Algorithm \ref{alg4}. We repeatedly apply our approximation algorithm from \Cref{thm:approx_2} on multiple datasets sampled from the Employee Access dataset under the same experimental setup except for the $\epsilon$-value which is now set to $0.001$. We see that the algorithm continues to perform much better than the guarantee of $2(\Delta + 1)(\log (n/\epsilon)+1)$ approximation provided in the analysis of the algorithm in \Cref{sec:delta_approx}. The results can be seen in \Cref{table1000}, \Cref{table2000}, \Cref{table3000}, and \Cref{table5000}.

\begin{table*}[t]
\centering
\begin{tabular}{|l|l|l|l|l|l|l|l|}
\hline
    Dataset & Sample Size & $\Delta$ & $\frac{UB}{SOL}$ & approx  & \shortstack[l]{no. of match- \\ ings} & \shortstack[l]{run-time \\ (seconds)}\\
\hline
    Employee Access data & 1000 &  3 & 5.43 & 191.2 & 892 & 11\\
\hline
    Employee Access data & 2000 & 3 & 7.24 & 196.8 & 1871 & 60\\
\hline
    Employee Access data & 3000& 4 & 9.19 & 250 & 2786 & 180\\
\hline
    Employee Access data & 5000 & 4 & 15.98 & 254 & 4651 & 900\\
\hline
    Grant Application Data & 8707 & 12 & 7.92 & 652 & 3836 & 540\\
\hline
\end{tabular}
\caption{Comparison of solution values on real-world datasets.}
\label{table1}
\end{table*}
    
\begin{table*}[t]
\centering
\begin{tabular}{|l|l|l|l|l|l|}
\hline
     $\Delta$ & $\frac{UB}{SOL}$ & approx  & \shortstack[l]{no. of mat- \\ chings} & \shortstack[l]{run-time \\ (seconds)}\\
\hline
    3 & 6.74 & 164.82 & 914 & 17.1\\
\hline
    3 & 3.39 & 164.58 & 908 & 17.8\\
\hline
    3 & 4.45 & 164.68 & 923 & 17.8\\
\hline
    3 & 6.34 & 164.63 & 897 & 16.6\\
\hline
    3 & 5.04 & 164.89 & 913 & 16.7\\
\hline
    3 & 3.78 & 164.83 & 905 & 20.4\\
\hline
    3 & 3.57 & 164.84 & 914 & 17.7\\
\hline
    3 & 5.4 & 164.76 & 904 & 17\\
\hline
    2 & 7.33 & 123.54 & 906 & 16.9\\
\hline
    2 & 4.11 & 123.6 & 935 & 17.7\\
\hline
\end{tabular}
\caption{Comparison of solution values with the theoretical bound for samples of size $1000$.}
\label{table1000}
\end{table*}

\begin{table*}
\centering
\begin{tabular}{|l|l|l|l|l|l|}
\hline
     $\Delta$ & $\frac{UB}{SOL}$ & approx  & \shortstack[l]{no. of mat- \\ chings} & \shortstack[l]{run-time \\ (seconds)}\\
\hline
    3 & 8.54 & 170.56 & 1852 & 105.5\\
\hline
    3 & 9.97 & 170.6 & 1863 & 117\\
\hline
    4 & 9.94 & 213.36 & 1850 & 105.1\\
\hline
    3 & 8.52 & 170.53 & 1879 & 108.7\\
\hline
    3 & 6.53 & 170.76 & 1864 & 109.1\\
\hline
    4 & 10.54 & 213.48 & 1868 & 113.3\\
\hline
    4 & 9.75 & 213.13 & 1867 & 108.1\\
\hline
    4 & 9.55 & 213.17 & 1865 & 105.5\\
\hline
    4 & 7.62 & 213.28 & 1848 & 107\\
\hline
    3 & 9.60 & 170.79 & 1846 & 103.1\\
\hline
\end{tabular}
\caption{Comparison of solution values with the theoretical bound for samples of size $2000$.}
\label{table2000}
\end{table*}

\begin{table*}
\centering
\begin{tabular}{|l|l|l|l|l|l|}
\hline
     $\Delta$ & $\frac{UB}{SOL}$ & approx  & \shortstack[l]{no. of mat- \\ chings} & \shortstack[l]{run-time \\ (seconds)}\\
\hline
    3 & 9.52 & 173.48 & 2828 & 328.7\\
\hline
    4 & 11.86 & 216.69 & 2793 & 384.4\\
\hline
    4 & 11.39 & 217.43 & 2816 & 318.7\\
\hline
    4 & 14.28 & 217.22 & 2825 & 288.5\\
\hline
    4 & 12.0 & 217.08 & 2798 & 317.4\\
\hline
    4 & 11.46 & 217.01 & 2797 & 322.7\\
\hline
    4 & 10.95 & 216.71 & 2787 & 313.5\\
\hline
    4 & 13.54 & 217.05 & 2805 & 290.45\\
\hline
    3 & 9.12 & 173.67 & 2791 & 374.6\\
\hline
    3 & 8.51 & 173.66 & 2760 & 314.7\\
\hline
\end{tabular}
\caption{Comparison of solution values with the theoretical bound for samples of size $3000$.}
\label{table3000}
\end{table*}

\begin{table*}
\centering
\begin{tabular}{|l|l|l|l|l|l|}
\hline
     $\Delta$ & $\frac{UB}{SOL}$ & approx  & \shortstack[l]{no. of mat- \\ chings} & \shortstack[l]{run-time \\ (seconds)}\\
\hline
    4 & 14.13 & 221.28 & 4671 & 1315.4\\
\hline
    4 & 13.28 & 220.78 & 4666 & 1332.6\\
\hline
    4 & 12.70 & 221.29 & 4670 & 1298.9\\
\hline
    4 & 18.81 & 221.14 & 4661 & 1292.6\\
\hline
    4 & 20.97 & 220.89 & 4617 & 1160.8\\
\hline
    5 & 13.76 & 265.67 & 4626 & 1494.9\\
\hline
    4 & 10.9 & 221.21 & 4646 & 1286.2\\
\hline
    5 & 11.95 & 265.23 & 4648 & 1295.6\\
\hline
    4 & 20.1 & 221.08 & 4609 & 1166.9\\
\hline
    5 & 9.88 & 265.09 & 4658 & 1229.5\\
\hline
\end{tabular}
\caption{Comparison of solution values with the theoretical bound for samples of size $5000$.}
\label{table5000}
\end{table*}
\section{Conclusion}
Various notions of group fairness and individual fairness in matching have been considered. However, to the best of our knowledge, this is the first work addressing both the individual and group fairness constraints in the same instance. Our work leads to several interesting open questions like improving the $O(\Delta\log{n})$ approximation ratio in Theorem~\ref{thm:approx_2} and extending our approximation results to the setting with lower bounds, and matching with two-sided preferences.
\section{Acknowledgements}
We thank Shankar Ram for the initial discussions. AP was supported in part by the Empower Programme of the Kotak IISc AI-ML Centre and the Walmart Center for Tech Excellence at IISc (CSR Grant WMGT-23-0001). AL was supported in part by SERB Award CRG/2023/002896, Pratiksha Trust Young Investigator Award and the Walmart Center for Tech Excellence at IISc (CSR Grant WMGT-23-0001). PN was supported by SERB Grant CRG/2019/004757.
\bibliographystyle{alpha}
\bibliography{references}
\appendix
\section{Appendix}

\subsection{Dealing with infeasibility}\label{sec:LPinfeasibility}
The introduction of individual fairness constraints introduces the challenge of having inconsistent constraints that could result in the LPs \ref{LP_Primal_IF} and \ref{LP_disjoint} becoming infeasible. To address this, one solution is to introduce a variable to calculate the smallest multiplicative relaxation
of the individual fairness constraints, ensuring the feasibility of both LPs.
Let $t \in [0, 1]$ be the additional variable. Given an instance of our problem, $\cI$, with $n$ items and $m$ platforms, we can formulate another LP where the lower bounds on all the individual fairness constraints are scaled by $t$, and the objective of this new LP is to maximize $t$.

\begin{lp}\label{LP_feasibility}
\begin{align}
&\max \quad t \\
\text{such that}\quad &t\cdot L_{a,k} \leq \sum_{p \in R_{a,k}}\cx \leq U_{a,k},  & \forall a \in A, \forall k \in [m] \\  
&l_p \leq \sum_{a \in N(p)} \cx \leq u_p , &\forall p \in P \\
&\cGL \leq \sum_{a \in A_h} \cx \leq \cGU , & \forall p \in P, \forall h \in [\chi] \label{LP-PB}\\
& 0 \le t \le 1 \\
&0 \leq \cx \leq 1 & \forall a \in A, \; \forall p\in P 
\end{align}
\end{lp}

Since we address strong group fairness (\Cref{def:strict_group_fairness}) only for disjoint groups (\Cref{disjoint_groups}), if \Cref{LP_feasibility} is infeasible, then there is no group-fair matching that satisfies the platform bounds (\Cref{LP-PB}). In this case, we just say that there is no group-fair matching. If \Cref{LP_feasibility} is feasible, let $t^*$ be its optimal solution, then scaling the lower bounds of all the individual fairness constraints of \Cref{LP_disjoint} by $t^*$ ensures feasibility.  We do not scale down the upper bounds because that would make the constraint tighter. 

It is easy to see how this method can be replicated to  \Cref{LP_Primal_IF} to ensure feasibility. Since the group fairness constraints only have upper bounds in \Cref{LP_Primal_IF}, a group-fair matching will always exist. Therefore, we can always compute a distribution over group-fair matchings such that the relaxed individual fairness constraints are satisfied with the approximation factor mentioned in our result \Cref{thm:approx_2}. Formally speaking, Algorithm \ref{alg4} computes a distribution $\cD$ over a set of group-fair matchings such that given the individual fairness parameters, $L_{a,k},U_{a,k} \in [0,1]$, for each item $a \in A$ and subset $R_{a,k}$ $\forall$ $k \in [m]$, \begin{align*}
      &\frac{1}{f_{\epsilon}}\left(t^* \cdot L_{a,k} - \epsilon\right) \leq \Pr_{\M \sim \cD}[\exists p\in R_{a,k} \textrm{ s.t. } (a,p)\in M] \\ &\leq  \frac{1}{f_{\epsilon}}\left(U_{a,k} + \epsilon\right).
\end{align*}

Here $f_{\epsilon} = \cO(\Delta\log (n/\epsilon))$, $n$ is the total number of items and each team can belong to at most $\Delta$ groups.

\subsection{Proof of \Cref{lem:integrality}}\label{appendix_integrality}
The integrality of the polytope of $GFLP$ is implicit in \cite{Rank-MaximalAndPOpular} where they construct a flow-network for the Classified Rank-Maximal Matching problem when the classes of each vertex form a laminar family. We provide an explicit proof of \Cref{lem:integrality} using the following Claim.

\begin{claim}\label{claim:integrality}\label{integrality}
In any basic feasible solution of $GFLP$, $\displaystyle \sum_{a \in N(p)} \cx$ is an integer, $\forall p \in P$ if the $l_p$, $u_p$, $\cGL$ and $\cGU$ values are integers, $\forall p \in P$, $\forall h \in [\chi]$.
\end{claim}
\begin{proof}
Let $\X$ be a basic feasible solution of $GFLP$. For an arbitrary platform, $p' \in P$, let there be $r$ groups in $C_{p'}$, say $\ch_1, \ch_2 \dots \ch_r$. 
Suppose $\displaystyle \sum_{a \in N(p')} x_{a,p'} = \sum_{i=1}^r \sum_{a \in \ch_i} x_{ap'}$ is not an integer. This implies that there exists at least one group, say $\ch_q \in C_{p'}$, such that $\displaystyle \sum_{a \in \ch_q}x_{a,p'}$ is fractional, which in turn implies that there is at least one item, say $b \in \ch_q$, such that $x_{bp'}$ is fractional. Let  
\begin{align*}
\begin{split}
    & w = \min\Big( x_{bp'}, \lceil x_{bp'} \rceil - x_{bp'}, \big\lceil \sum_{a \in \ch_q}x_{a,p'} \big\rceil - \sum_{a \in \ch_q}x_{a,p'}, \\& \big\lceil \sum_{a \in N(p')} x_{a,p'} \big\rceil - \sum_{a \in N(p')} x_{a,p'}, 
    \sum_{a \in \ch_q}x_{a,p'} - \big\lfloor \sum_{a \in \ch_q}x_{a,p'} \big\rfloor \\ & ,\sum_{a \in N(p')} x_{a,p'} - \big\lfloor \sum_{a \in N(p')} x_{a,p'} \big\rfloor \Big).
\end{split}
\end{align*}
Since $x_{a,p'}$, $\sum_{a \in \ch_q}x_{a,p'}$, and $\sum_{a \in N(p')} x_{a,p'}$ are not integers by our assumption, $w \in (0,1)$. Let us modify $\X$ by replacing $x_{bp'}$ with $x_{bp'}+w$, and let the resulting ordered set be $\cy$. By the definition of $w$ and the assumption that $\forall p \in P$, $\forall h \in [\chi]$, the $u_p$ and $\cGU$ values are integers, $\cy$ doesn't violate the constraints \ref{LP-4} to \ref{LP-8}. Similarly, since $\forall p \in P$, $\forall h \in [\chi]$, the $l_p$ and $\cGL$ values are assumed to be integers, if we modify $\X$ by replacing $x_{bp'}$ with $x_{bp'}-w$, the resulting ordered set, say $\cz$, will also not violate the constraints \ref{LP-4} to \ref{LP-8}. Hence $\cy$ and $\cz$ are feasible solutions of $GFLP$. Clearly, 
$$\X = \frac{1}{2}\cy+\frac{1}{2}\cz,$$ which is a contradiction since a basic feasible solution of any LP cannot be written as a convex combination of two other points in the polytope of the same LP.
\end{proof}

\begin{proof}[Proof of \Cref{lem:integrality}]
Let $\X$ be a basic feasible solution of $GFLP$. Let us suppose that $\X$ is fractional. From \Cref{claim:integrality}, we know that $\displaystyle \sum_{a \in N(p)} \cx$ is an integer, $\forall p \in P$, in any vertex solution of $GFLP$. Therefore, for some arbitrary platform, say $p' \in P$, if there is an edge, say $\cE{b}{p'}$ where $b \in A$, such that $x_{bp'}$ is fractional, then there must be at least one other edge $\cE{b'}{p'}$ where $b' \in A$, such that $x_{b'p'}$ is also fractional. Let $b \in \ch_1$ and $b' \in \ch_2$ where $\ch_1, \ch_2 \in C_{p'}$ and let $x_{bp'} > x_{b'p'}$ without loss of generality. 
\begin{align*}
 \begin{split}
    & w = \min \Big(1-x_{bp'}, x_{b'p'}, \sum_{a \in \ch_1}x_{ap'} - \Big\lfloor\sum_{a \in \ch_1}x_{ap'}\Big\rfloor,  \\& \sum_{a \in \ch_2}x_{ap'} - \Big\lfloor\sum_{a \in \ch_2}x_{ap'}\Big\rfloor, \Big\lceil\sum_{a \in \ch_1}x_{ap'} \Big\rceil - \sum_{a \in \ch_1}x_{ap'}, \\ &\Big\lceil\sum_{a \in \ch_2}x_{ap'}\Big\rceil - \sum_{a \in \ch_2}x_{ap'} \Big)
\end{split}
\end{align*}
Let us modify $\X$ by replacing $x_{bp'}$ and $x_{b'p'}$ with $x_{bp'}+w$ and $x_{b'p'} - w$, respectively, and let the resulting ordered set be $\cy$. By the definition of $w$ and the assumption that $\forall p \in P$, $\forall h \in [\chi]$, the $\cGU$ values are integers, $\cy$ doesn't violate the constraints \ref{LP-6} to \ref{LP-8}. Similarly, if we modify $\X$ by replacing $x_{bp'}$ and $x_{b'p'}$ with $x_{bp'}-w$ and $x_{b'p'} + w$, respectively, the resulting ordered set, say $\cz$, will also not violate the constraints \ref{LP-6} to \ref{LP-8}. It is easy to see that $$\displaystyle \sum_{a \in N(p')} y_{ap'} = \sum_{a \in N(p')} z_{ap'} = \sum_{a \in N(p')} x_{ap'}$$
Hence $\cy$ and $\cz$ also satisfy constraint \ref{LP-4} and hence are feasible solutions of $GFLP$. Clearly, 
$$\X = \frac{1}{2}\cy+\frac{1}{2}\cz,$$ which is a contradiction since a basic feasible solution of any LP cannot be written as a convex combination of two other points in the polytope of the same LP.
\end{proof}

\section{Extension to other notions of fairness}\label{appendix:extension}
Before delving into how to extend our results to the fairness notions in \Cref{subsec:fair_notions}, let us look at a more generic definition of individual fairness that provides a framework to accommodate various individual fairness settings, including the one in our problem (\Cref{def:ind_fairness}). 

\begin{definition}[{\bf Generic Probabilistic Individual Fairness}]
    Given {\em individual fairness parameters} $L_{a,S} \in [0,1]$ and $U_{a,S} \in [0,1]$ for each item $a$ and each subset $S$ of $N(a)$, where $N(a)$ denotes the neighborhood of item $a$ in $G$. A distribution $\cD$ on matchings in $G$ is {\em probabilistic individually fair} if and only if $\forall a\in A,\forall S \subseteq N(a)$,
\begin{equation}\label{eq:GenIndFairness}
\begin{aligned}
    & L_{a,S} \leq \Pr_{M \sim \cD}[\exists p\in S \textrm{ s.t. } (a,p)\in M]  \leq U_{a,S} \text{ }
\end{aligned}
\end{equation}
\end{definition}

It is easy to see how \Cref{eq:GenIndFairness} can not only capture the probabilistic individual fairness constraints in our problem (\Cref{def:ind_fairness}) but also capture the requirement that items are matched to a low-ranking platform in their preference list with low probability. This model allows users to set Individual Fairness constraints based on their requirements. All our results in Theorems \ref{thm:approx_2}, \ref{thm:approx_1}, \ref{thm:approx_3} and \ref{thm:exact_algo} can be extended to this notion of fairness by simply replacing the probabilistic individual fairness constraints in \Cref{LP_Primal_IF} and \Cref{LP_disjoint} with \Cref{eq:GenIndFairness}. Our results provide the same approximation guarantee for the {\em Generic Probabilistic Individual Fairness} as the {\em Probabilistic Individual Fairness}. The only difference is that the algorithms may not stay polynomial in the number of nodes and edges of the input graph $G$ due to the potentially exponential number of constraints in the LPs. However, the runtime will stay polynomial in the number of subsets with non-trivial bounds in the input instance. 

Next, we will look at extending our results to the fairness notions in \Cref{subsec:fair_notions}. Let the input instance also provide the lower bound on the size of the resulting matching, say $\zeta$.

\begin{algorithm}[tb]\label{alg:updateLP}
\caption{Update-LP$(LP, z, \zeta)$}
\label{algmod}
\textbf{Input} :  $LP, z, \zeta$ \\
\textbf{Output} : Modified $LP$ \\
\begin{algorithmic}[1]
\STATE Update the objective of $LP$ to $$\max z$$ \\
\STATE Add the following constraint to $LP$ $$\sum_{(a,p) \in E} x_{ap} \geq \zeta$$\\
\STATE Return $LP$
\end{algorithmic}
\end{algorithm}

\subsection{Maxmin group fairness}\label{appendix_maxmin_extension}
In order to express {\em Maxmin group fairness} into \Cref{LP_disjoint}, we first execute \hyperref[alg:updateLP]{Update-LP}(\Cref{LP_disjoint}, $\mu, \zeta$). Let us call the resulting LP, $MAXLP$. Now we update constraint \ref{LP-6} in $MAXLP$ to
    \begin{equation}\label{lp_maxmin}
    \sum_{(a,p) \in E} x_{ap} - \mu \geq 0, \forall p \in P, \forall h \in [\chi] 
    \end{equation}
Equation \ref{lp_maxmin} ensures that even the group with the minimum representation in any platform has a representation at least $\mu$ and with an objective of maximizing the value of $\mu$, $MAXLP$, returns a $\mu$ value that maximizes the representation of the least represented group subject to other constraints. 
Let $\X^*, \mu^*$ denote an optimal solution of $MAXLP$, then $\X^*$ is a feasible solution of \Cref{LP_disjoint} if we set $\cGL = \mu$, $\forall h \in [\chi], \forall p \in P$.
By replacing Step 1 in Algorithm \ref{alg1} with `Solve $MAXLP$ on $G$ and store the result in $\X, \mu$' we can obtain the results from \Cref{thm:exact_algo} in this setting but the distribution is over a set of {\em \maxmin group fair} matching. This proves part of $(1)$ in \Cref{thm:ext_disjoint}.

\subsection{Maxmin individual fairness}\label{appendix_maxmin_extension_ind}
Let \Cref{LP-2} be changed to 
$$\sum_{p \in P}\cx \geq \mu, \forall a \in A$$
Therefore, {\em Maxmin individual fairness} constraint can now be expressed into \Cref{LP_disjoint} and \Cref{LP_Primal} by executing $MMLP = $ \hyperref[alg:updateLP]{Update-LP}(\Cref{LP_disjoint}, $\mu, \zeta$) and $MMLP' =$ \hyperref[alg:updateLP]{Update-LP}(\Cref{LP_Primal}, $\mu, \zeta$), respectively. Under this setting, where \Cref{LP-2} has been updated, if $\X^*, \mu^*$ denote an optimal solution of $MMLP$, then $\X^*$ is a feasible solution of \Cref{LP_disjoint}. By replacing Step 1 in Algorithm \ref{alg1} with `Solve $MMLP$ on $G$ and store the result in $\X, \mu$' we can obtain the results from \Cref{thm:exact_algo} in this setting. This proves $(2)$ of \Cref{thm:ext_disjoint}.

Similarly if $\cy^*, z^*$ denote an optimal solution of $MMLP'$, then $\cy^*$ is a feasible solution of \Cref{LP_Primal}.
By replacing Step 1 in Algorithm \ref{alg4}, and Algorithm \ref{alg2} with `Solve $MMLP'$ on $G$ and store the result in $\X, \mu$' we can obtain the results from \Cref{thm:approx_2}, and \Cref{thm:approx_1} in this setting where any violation of individual fairness would be a violation of {\em Maxmin individual fairness}. This proves part of $(1)$ and $(2)$ in \Cref{thm:ext}

\subsection{Mindom group fairness}\label{appendix_minmax_extension}
{\em Mindom group fairness} can be expressed into \Cref{LP_disjoint} and \Cref{LP_Primal} by first executing $MINLP = $ \hyperref[alg:updateLP]{Update-LP}(\Cref{LP_disjoint}, $-\mu, \zeta$) and $MINLP' =$ \hyperref[alg:updateLP]{Update-LP}(\Cref{LP_Primal}, $-\mu, \zeta$), respectively and then by updating constraint \ref{LP-6} and \ref{LP-14} in $MINLP$ and $MINLP'$ respectively to 
\begin{equation}\label{lp:minmax}
    \sum_{(a,p) \in E} x_{ap} - \mu \leq 0, \forall p \in P, \forall h \in [\chi] 
\end{equation}
Let $\X^*, \mu^*$ denote an optimal solution of $MINLP$, then $\X^*$ is a feasible solution of \Cref{LP_disjoint} if we set $\cGU = \mu$, $\forall h \in [\chi], \forall p \in P$. Therefore, by replacing Step 1 in Algorithm \ref{alg1} with `Solve $MINLP$ on $G$ and store the result in $\X, \mu$' we can obtain the results from \Cref{thm:exact_algo} in this setting but the distribution is over a set of {\em \minmax group fair} matching. This proves the remaining parts of $(1)$ in \Cref{thm:ext_disjoint}.

Similarly, if $\cy^*, z^*$ denote an optimal solution of $MINLP'$, then $\cy^*$ is a feasible solution of \Cref{LP_Primal} if we set $\cGU = \mu$, $\forall h \in [\chi], \forall p \in P$.
Therefore, by replacing Step 1 in Algorithm \ref{alg4}, and Algorithm \ref{alg2} with `Solve $MINLP'$ on $G$ and store the result in $\X, \mu$' we can obtain the results from \Cref{thm:approx_2}, and \Cref{thm:approx_1} in this setting but the distribution is over a set of {\em \minmax group fair} matching for \Cref{thm:approx_2} and any matching in the distribution violates the {\em \minmax group fairness} condition by an additive factor of at most $(2 - \lambda)\Delta$ for \Cref{thm:approx_1}. This proves the remaining parts of $(1)$ and $(2)$ in \Cref{thm:ext}.

\end{document}